\documentclass{article}

\pdfoutput=1

\usepackage{microtype}
\usepackage{graphicx}
\usepackage{subfigure}
\usepackage{booktabs} 

\usepackage{float}

\usepackage[dvipsnames]{xcolor}
\usepackage{hyperref}

\usepackage[accepted]{icml2020}

\usepackage[]{natbib}
\usepackage{url}            
\usepackage{booktabs}       
\usepackage{amsfonts}       
\usepackage{nicefrac}       
\usepackage{microtype}      
\usepackage{amsmath,amsthm}
\usepackage{mathtools}
\usepackage{enumitem}
\usepackage{tikz}
\usetikzlibrary{arrows.meta}
\usetikzlibrary{backgrounds, calc}

\usepackage{algorithm}
\usepackage{algorithmic}

\usepackage{dsfont}
\usepackage{amssymb}
\usepackage{upgreek}
\usepackage{multirow}
\usepackage{soul}
\usepackage{xspace}

\usepackage{pifont}

\let\originalleft\left
\let\originalright\right
\renewcommand{\left}{\mathopen{}\mathclose\bgroup\originalleft}
\renewcommand{\right}{\aftergroup\egroup\originalright}
\newcommand{\wt}[1]{\widetilde{#1}}
\newcommand{\wb}[1]{\overline{#1}}

\newcommand{\wh}[1]{\widehat{#1}}

\newcommand{\RgVal}{{\small\textsc{Rg-Val}}\xspace}
\newcommand{\UCSSP}{{\small\textsc{UC-SSP}}\xspace}
\newcommand{\UCSSPBF}{{\small\textsc{UC-SSP-BF}}\xspace}
\newcommand{\UCLSSP}{{\small\textsc{UC-LSSP}}\xspace}
\newcommand{\UCLSSPBF}{{\small\textsc{UC-LSSP-BF}}\xspace}

\newcommand{\UCBVI}{{\small\textsc{UCBVI}}\xspace}

\newcommand{\SSP}{{\small\textsc{SSP}}\xspace}
\newcommand{\EVI}{{\small\textsc{EVI}}}
\newcommand{\VI}{{\small\textsc{VI}}\xspace}
\newcommand{\maybeEVI}{{\small\textsc{(E)VI}}}

\newcommand{\RESET}{{\small\textsc{RESET}}\xspace}

\newcommand{\UCBVICH}{{\small\textsc{UCBVI-CH}}\xspace}
\newcommand{\UCBVIBF}{{\small\textsc{UCBVI-BF}}\xspace}

\newcommand{\UCRL}{{\small\textsc{UCRL}}\xspace}
\newcommand{\UCRLSSP}{{\small\textsc{UCRL-SSP}}\xspace}
\newcommand{\UCRLtwo}{{\small\textsc{UCRL2}}\xspace}
\newcommand{\UCRLB}{{\small\textsc{UCRL2B}}\xspace}
\newcommand{\TUCRL}{{\small\textsc{TUCRL}}\xspace}
\newcommand\myeqa{\mathrel{\stackrel{\makebox[0pt]{\mbox{\normalfont\tiny (a)}}}{=}}}
\newcommand\myeqb{\mathrel{\stackrel{\makebox[0pt]{\mbox{\normalfont\tiny (b)}}}{=}}}

\newcommand\myineeqb{\mathrel{\stackrel{\makebox[0pt]{\mbox{\normalfont\tiny (b)}}}{\leq}}}
\newcommand\myineeqc{\mathrel{\stackrel{\makebox[0pt]{\mbox{\normalfont\tiny (c)}}}{\leq}}}

\newcommand{\SD}{\Pi^\text{SD}}
\newcommand{\PSD}{\Pi^\text{PSD}}

\makeatletter
\newcommand\footnoteref[1]{\protected@xdef\@thefnmark{\ref{#1}}\@footnotemark}
\makeatother

\DeclarePairedDelimiter\abs{\lvert}{\rvert}%
\DeclarePairedDelimiter\norm{\lVert}{\rVert}%

\newtheorem{theorem}{Theorem}
\newtheorem{lemma}{Lemma}
\newtheorem{proposition}{Proposition}

\theoremstyle{definition}
\newtheorem{definition}{Definition}
\newtheorem{assumption}{Assumption}
\newtheorem*{remark}{Remark}

\DeclareMathOperator*{\argmax}{arg\,max}
\DeclareMathOperator*{\argmin}{arg\,min}

\usepackage[colorinlistoftodos, textwidth=18mm]{todonotes}

\newcommand{\CommentSty}[1]{\texttt{#1}}
\newcommand{\tcp}[1]{\texttt{// #1}}

\newcommand{\jt}[1]{\textcolor{black}{#1}}

\icmltitlerunning{No-Regret Exploration in Goal-Oriented Reinforcement Learning}

\begin{document}

\twocolumn[
\icmltitle{No-Regret Exploration in Goal-Oriented Reinforcement Learning}



\icmlsetsymbol{equal}{*}

\begin{icmlauthorlist}
\icmlauthor{Jean Tarbouriech}{fb,inria}
\icmlauthor{Evrard Garcelon}{fb}
\icmlauthor{Michal Valko}{inria}
\icmlauthor{Matteo Pirotta}{fb}
\icmlauthor{Alessandro Lazaric}{fb}
\end{icmlauthorlist}

\icmlaffiliation{fb}{Facebook AI Research, Paris, France}
\icmlaffiliation{inria}{SequeL team, Inria Lille - Nord Europe, France}

\icmlcorrespondingauthor{Jean Tarbouriech}{jean.tarbouriech@gmail.com}


\vskip 0.3in
]



\printAffiliationsAndNotice{}  

\begin{abstract}
Many popular reinforcement learning problems (e.g., navigation in a maze, some Atari games, mountain car) are instances of the \textit{episodic setting} under its \textit{stochastic shortest path} (SSP) formulation, where an agent has to achieve a goal state while minimizing the cumulative cost. Despite the popularity of this setting, the exploration-exploitation dilemma has been sparsely studied in general SSP problems, with most of the theoretical literature focusing on different problems (i.e., finite-horizon and infinite-horizon) or making the restrictive \textit{loop-free} SSP assumption (i.e., no state can be visited twice during an episode). In this paper, we study the general SSP problem with no assumption on its dynamics (some policies may actually never reach the goal). We introduce \UCSSP, the first no-regret algorithm in this setting, and prove a regret bound scaling as $\displaystyle \widetilde{\mathcal{O}}( D S \sqrt{ A D K})$ after $K$ episodes for any unknown SSP with $S$ states, $A$ actions, positive costs and \textit{SSP-diameter} $D$, defined as the smallest expected hitting time from any starting state to the goal. We achieve this result by crafting a novel stopping rule, such that \UCSSP may interrupt the current policy if it is taking \textit{too long} to achieve the goal and switch to alternative policies that are designed to \textit{rapidly} terminate the episode.
\end{abstract}


\section{Introduction}
\label{introduction}

We consider the problem of exploration-exploitation in episodic Markov decision processes (MDPs), where the objective is to minimize the expected cost to reach a specific goal state. Several popular reinforcement learning (RL) problems fall into this framework, such as navigation problems, many \emph{Atari games} (e.g., breakout) and Mujoco environments (e.g., reacher). In all these problems, the length of an episode (i.e., the time to reach the goal state) is unknown and depends on the policy executed during the episode. Furthermore, the performance is not directly connected to the length of the episode, as the objective is to minimize the cost over time rather than reaching the goal state as fast as possible. The conditions for the existence and the computation of an optimal policy have been studied in the MDP literature under the name of the \emph{stochastic shortest path} (SSP) problem~\citep[][Sect.\,3]{bertsekas1995dynamic}. 

The exploration-exploitation dilemma has been extensively studied in the finite-horizon (see e.g., \citealp{azar2017minimax, zanette2019tighter}) and infinite-horizon settings (see e.g., \citealp{jaksch2010near, fruit2018near, fruit2018efficient}). In the former, the performance is optimized over a fixed and known horizon of $H$ steps. Typically, this model is used to solve SSP problems by setting $H$ \textit{large enough}. While for $H\rightarrow \infty$ the optimal finite-horizon policy converges to the optimal SSP policy, for any finite $H$, this approach may introduce a bias leading exploration algorithms to converge to suboptimal policies and suffer linear regret (see e.g., \citealp{toromanoff2019deep}, for a discussion of this problem in Atari games). In the latter, the performance is optimized for the asymptotic average cost. While this removes any strict ``deadline'', it does not introduce any incentive to reach the goal state. This may favor policies with small average cost and yet poor performance in the SSP sense, as they may never terminate. \jt{Note that SSP forms an important class of MDPs as both infinite-horizon (discounted) and finite-horizon MDPs, two much more extensively researched settings, are a subtype of SSP-MDPs~\cite{bertsekas1995dynamic, guillot2020stochastic}.}

Prior work on exploration in SSPs can be divided in two cases.
The first is the online shortest path routing problem, which has deterministic dynamics and stochastic rewards. In this case, the optimal policy is open-loop (i.e., it is a sequence of actions independent from the states) and it can be solved as an instance of a combinatorial bandit problem~\citep[see e.g.,][]{gyorgy2007line, talebi2017stochastic}. Exploration algorithms know the set of admissible paths of bounded length and regret bounds are available in both the semi- and full-bandit setting. 
The second case allows for stochastic transitions and mostly considers adversarial problems, but it is restricted to \textit{loop-free} environments~\citep[see e.g.,][]{jin2019learning, rosenberg2019online, rosenberg2019online2, neu2012adversarial, neu2010online, zimin2013online}. Under this assumption, the state space can be decomposed into $L$ non-intersecting layers $X_0, \ldots, X_L$ such that $X_0 = \{x_0\}$ and $X_L = \{x_L\}$, and transitions are only possible between consecutive layers. In this case, it is possible to derive regret bounds leveraging the fact that \textit{any} episode length is upper bounded by $L$ almost surely. Unfortunately, this requirement is restrictive and fails to hold in many realistic environments. 

In this paper, exploration in general SSP problems is investigated for the first time. The solution of an SSP is obtained by computing the policy minimizing the value function, i.e., the expected costs accumulated until reaching the goal state. Studying SSP value functions poses technical difficulties that do not appear in the conventional settings such as loop-free SSP, finite-horizon and infinite-horizon: \textbf{1)} it features two possibly conflicting objectives: quickly reaching the goal state while minimizing the costs along the way; \textbf{2)} it is unbounded for policies that may never reach the goal state (i.e., non-proper policies); \textbf{3)} it is not state-independent (a crucial property of the gain of any optimal policy in infinite-horizon); \textbf{4)} its number of summands may differ from one trajectory to another due to variations in the time to reach the goal state (thus making the regret decomposition tricky compared to finite-horizon); \textbf{5)} it cannot be computed using backward induction (a crucial technique used in finite-horizon)\jt{; \textbf{6)} it cannot be discounted (since a discount factor would have a undesirable effect of biasing importance towards short-term behavior and thus weakening the incentive to eventually reach the goal state). This last point means that SSP-MDPs do not have a notion of ``equivalent horizon'', which is $1/(1-\gamma)$ in the special case of infinite-horizon discounted MDPs with known discount factor $\gamma$, thus making the general setting of SSP-MDPs more difficult to analyze.}

While we leverage algorithmic and technical tools from both finite- and infinite-horizon settings, 
tackling the general SSP problem requires introducing novel techniques to manage the challenges highlighted above. Notably, we investigate the properties of \textit{optimistic} policies and their associated \textit{discrete phase-type distributions} (i.e., the hitting time distribution) to design a novel criterion to stop executing the current optimistic SSP policy \textit{during} an episode and switch to alternative policies designed to rapidly reach the goal.

The main contributions of this paper are: \textbf{1)} We formalize \emph{exploration-exploitation} in SSP problems by defining an adequate notion of regret (Sect.\,\ref{section_SSP}). \textbf{2)} We show that the special case of SSP with uniform costs can be cast as an infinite-horizon problem and tackled by \UCRLtwo~\cite{jaksch2010near} with a regret bound adapting to the complexity of the environment (Sect.\,\ref{section_uniform_costs}). \textbf{3)} We then introduce \UCSSP, the first algorithm with vanishing regret in general SSP problems (Sect.\,\ref{section_general_SSP}). We also show that not only \UCSSP effectively deals with the general case, but it remains competitive (if not better) even in the limit cases of uniform costs or loop-free SSP, which can be addressed by infinite- and finite-horizon regret minimization algorithms respectively. \textbf{4)} Moreover, we demonstrate how our (mild) assumptions (e.g., no dead-end states, positive costs) can be effectively relaxed using variants of \UCSSP (Sect.\,\ref{section_relaxation_asm}). \jt{Finally, we support our theoretical findings with experiments in App.\,\ref{app:experiments}.} 




\section{Stochastic Shortest Path (SSP)}
\label{section_SSP}

We consider a finite \textit{stochastic shortest path} problem~\citep[][Sect.\,3]{bertsekas1995dynamic} $M := \langle\mathcal{S}', \mathcal{A}, c, p, s_{0} \rangle$, where $\mathcal{S}' :=  \mathcal{S}\cup \{\overline{s}\}$ is the set of states with $\overline{s}$ being the goal state (also called the terminal state) and $s_{0} \in \mathcal{S}$ being the starting state\footnote{\label{note1} Our algorithm can handle any (possibly unknown) distribution of initial states.}, and $\mathcal{A}$ is the set of actions. We denote by $A = \abs{\mathcal{A}}$ and $S = \abs{\mathcal{S}}$ the number of actions and non-goal states. Each state-action pair $(s,a) \in \mathcal{S} \times \mathcal{A}$ is characterized by a known, deterministic cost $c(s,a)$ and an unknown transition probability distribution $p(\cdot \, \vert \, s,a)$ over next states. The goal state $\overline{s}$ is absorbing (i.e.,~$p(\overline{s} \, \vert \, \overline{s},a) = 1$ for all $a \in \mathcal{A}$) and cost-free (i.e.,~$c(\overline{s},a) = 0$ for all $a \in \mathcal{A}$). We assume the following property of the cost function.

\begin{assumption}\label{assumption_costs}
There exist known constants $0 < c_{\min} \leq c_{\max}$ such that $c(s,a) \in [c_{\min}, c_{\max}]$ for all $(s,a) \in \mathcal{S} \times \mathcal{A}$.
\end{assumption}

Extending the setting to unknown, stochastic costs poses no major difficulty, as long as the learner knows in advance the range of the costs, i.e., the constants $c_{\min}$ and $c_{\max}$ (see  App.\,\ref{subsection_relaxation_unknown_costs}). Moreover, in Sect.\,\ref{section_relaxation_asm} we derive a variant of our algorithm that can handle zero costs (i.e., $c_{\min} = 0$).

\citet{bertsekas1995dynamic} showed that under Asm.\,\ref{assumption_costs} we can restrict the attention to the set of stationary deterministic policies $\SD := \{\pi : \mathcal{S} \to \mathcal{A}\}$. For any $\pi \in \SD$ and $(s,s') \in \mathcal{S} \times \mathcal{S}'$, the (possibly unbounded) \textit{hitting time} to $s'$ starting from $s$ is denoted by $\tau_{\pi}(s \rightarrow s') := \inf \{ t \geq 0: s_{t+1} = s' \,\vert\, s_1 = s, \pi \}$. We also set $\tau_{\pi}(s) := \tau_{\pi}(s \rightarrow \overline{s})$. 

\begin{assumption}\label{assumption_finite_SSP_diameter}
We define the \textit{SSP-diameter} $D$ as
\begin{align}\label{definition_SSP_diameter}
    D := \max_{s \in \mathcal{S}} \min_{\pi \in \SD} \mathbb{E}\left[\tau_{\pi}(s)\right],
\end{align}
and we assume that $D < +\infty$.
\end{assumption}

We say that $M$ is \textit{SSP-communicating} when Asm.\,\ref{assumption_finite_SSP_diameter} holds. We defer to Sect.\,\ref{section_relaxation_asm} the treatment of the case $D = +\infty$.

The \textit{value function} (also called expected cost-to-go) of any $\pi \in \SD$ is defined as
\begin{align*}
    V^{\pi}(s_{0}) 
    := \mathbb{E}\bigg[ \sum_{t = 1}^{\tau_{\pi}(s_0)} c(s_{t}, \pi(s_t)) \,\Big\vert\,s_{0}\bigg].
\end{align*}
For any vector $V \in \mathbb{R}^S$, the optimal Bellman operator is defined as
\begin{align*}
\mathcal{L}V(s) := \min_{a \in \mathcal{A}} \Big\{ c(s, a) + \sum_{y \in \mathcal{S}} p(y  \, \vert \,  s,a) V(y) \Big\}.
\end{align*}
An important role in the definition of the SSP is played by the set $\PSD \subseteq \SD$ of proper stationary policies. 
\begin{definition}
	A stationary policy $\pi$ is \textit{proper} if $\overline{s}$ is reached with probability 1 from any state in $\mathcal{S}$ following $\pi$.\footnote{Note that Def.\,\ref{def_proper} is slightly different from (and is implied by) the conventional definition of \citet[][Sect.\,3.1]{bertsekas1995dynamic}, for which a policy is proper if there is a positive probability that $\overline{s}$ will be reached after at most $S$ stages.}
	\label{def_proper}
\end{definition}

The next lemma shows that the SSP problem is well-posed.
\begin{lemma}
Under Asm.\,\ref{assumption_costs} and \ref{assumption_finite_SSP_diameter}, there exists an optimal policy $\pi^{\star} \in \arg\min_{\pi \in \PSD} V^{\pi}(s_{0})$ for which $V^\star = V^{\pi^\star}$ is the unique solution of the optimality equations $V^\star = \mathcal{L} V^\star$ and $V^\star(s) < +\infty$ for any $s \in \mathcal{S}$.
\label{lemma_wellposedproblem}
\end{lemma}


Similarly to the average-reward case, we can provide a bound on the range of the optimal value function depending on the largest cost and the SSP-diameter.
\begin{lemma}
Under Asm.\,\ref{assumption_costs} and~\ref{assumption_finite_SSP_diameter}, 
$\norm{V^{\star}}_{\infty} \leq c_{\max} D$.
\label{lemma_1}
\end{lemma}

For any $\pi \in \PSD$, its (almost surely finite) hitting time starting from any state in $\mathcal{S}$ follows a \textit{discrete phase-type distribution}, or in short \textit{discrete PH distribution} (see e.g.,~\citealp[][Sect.\,2.5]{latouche1999introduction} for an introduction). Indeed, its induced Markov chain is terminating with a single absorbing state $\overline{s}$ and all the other states are transient. The transition matrix associated to $\pi$, denoted by $P_{\pi} \in \mathbb{R}^{(S+1) \times (S+1)}$, can thus be arranged in the following canonical form
\[
P_{\pi}=
\left[
\begin{array}{cc}
Q_{\pi} & R_{\pi} \\
0 & 1
\end{array}
\right],
\]
where $Q_{\pi} \in \mathbb{R}^{S \times S}$ is the transition matrix between non-absorbing states (i.e., $\mathcal{S}$) and $R_{\pi} \in \mathbb{R}^{S}$ is the transition vector from $\mathcal{S}$ to $\overline{s}$. Note that $Q_{\pi}$ is strictly substochastic ($Q_{\pi} \mathds{1} \leq \mathds{1}$
where $\mathds{1} := (1,\ldots,1)^T \in \mathbb{R}^S$ and $\exists j$ s.t.\,$(Q_{\pi} \mathds{1})_j < 1$). Denoting by $\mathds{1}_s$ the $S$-sized one-hot vector at the position of state $s \in \mathcal{S}$, we have the following result (see e.g.,~\citealp[][Thm.\,2.5.3]{latouche1999introduction}).

\begin{proposition} For any $\pi \in \PSD$, $s \in \mathcal{S}$ and $n > 0$,
\begin{align*}
\mathbb{P}(\tau_{\pi}(s) > n) = \mathds{1}_s^\top Q_{\pi}^n \mathds{1} = \sum_{s' \in \mathcal{S}} (Q_{\pi}^n)_{s s'}.
\end{align*}
\vspace{-0.1in}
\label{equation_discrete_phase_type_distribution}
\end{proposition}

Finally, for any $X \in \mathbb{R}^{m \times n}$ we define the $\infty$-matrix-norm $\norm{X}_{\infty} := \max_{ 1 \leq i \leq m} \sum_{j=1}^n \abs{X_{ij}}$.

\textbf{Learning problem.}
We consider the learning problem where $\mathcal{S}', \mathcal{A},$ and $c$ are known, while the dynamics $p$ is unknown and can be estimated online. An \textit{environmental episode} starts at $s_0$ and ends \textit{only} when the goal state $\overline{s}$ is reached. We evaluate the performance of an algorithm $\mathfrak{A}$ after $K$ environmental episodes by its cumulative \textit{SSP-regret}
\begin{align*}
\Delta(\mathfrak{A},K) := \sum_{k = 1}^{K} \bigg[ \Big( \sum_{h=1}^{\tau_k(s_0)} c(s_{k,h}, \mu_k(s_{k,h})) \Big) - V^{\star}(s_{0}) \bigg],
\end{align*} 
where for any $k \in [K]$,\footnote{For any integer $n$, we denote by $[n]$ the set $\{ 1, \ldots, n\}$.} $\tau_k(s_0)$ is the length of episode $k$ following a possibly non-stationary policy $\mu_k=(\pi_{k,0}, \pi_{k,1}, \pi_{k,2}, \ldots)$, $\pi_{k,i} \in \SD$, until $\overline{s}$ is reached. Moreover, $s_{k,h}$ denotes the $h$-th state visited during episode $k$. $\Delta(\mathfrak{A}, K)$ also corresponds to the cumulative SSP-regret after $T_K$ steps, where $T_K := \sum_{k=1}^K \tau_k(s_0)$ is the time step at the end of episode $K$. This definition resembles the infinite-horizon regret, where the performance of the algorithm is evaluated by the costs accumulated by executing $\mu_k$. At the same time, it incorporates the episodic nature of finite-horizon problems, where the performance of the optimal policy is evaluated by its value function at the initial state. Nonetheless, notice that we cannot use the finite-horizon regret definition, i.e., $\sum_{k=1}^K V^{\mu_k}(s_0) - V^{\star}(s_0)$, where a policy $\mu_k$ is chosen at the beginning of the episode and run until its termination. Indeed, as $\mu_k$ may be non-proper and satisfy $V^{\mu_k}(s_0) = + \infty$, the execution of a single non-proper policy would directly lead to an unbounded regret.

\section{Uniform-cost SSP}
\label{section_uniform_costs}

In this section we focus on the SSP problems with uniform costs to illustrate a very first case where a sublinear regret can be achieved without any restrictive loop-free assumption. In particular, we show that in this case the SSP problem can be cast as an infinite-horizon problem and that an algorithm such as \UCRLtwo~\citep{jaksch2010near} can be directly applied and achieve surprisingly good regret guarantees.

\begin{assumption}[\textbf{only in Sect.\,\ref{section_uniform_costs}}]
The costs $c(s,a)$ are constant (equal to $1$ w.l.o.g.) for all $(s,a) \in \mathcal{S} \times \mathcal{A}$. 
\label{assumption_uniform_costs}
\end{assumption}

In this case, solving the SSP problem corresponds to computing the policy minimizing the expected hitting time to the goal $\overline{s}$.

We introduce the infinite-horizon reward-based MDP $M_{\infty} := \langle \mathcal{S}', \mathcal{A}, r_{\infty}, p_{\infty}, s_0 \rangle$, with reward $r_{\infty} = \mathds{1}_{\overline{s}}$ and $p_{\infty}(\cdot \, \vert \,  s,a) = p(\cdot \, \vert \,  s,a)$ for $s \neq \overline{s}$ and $p_{\infty}(\cdot \, \vert \,  \overline{s},a) = \mathds{1}_{s_0}$ for all $a$. In words, the transitions in $M_{\infty}$ behave as in $M$ and give zero rewards except at $\overline{s}$ where all actions give a reward of 1 and loop back to $s_0$ instead of self-looping with probability 1. We show that the solution of $M_{\infty}$ coincides with solving the original SSP and we bound the SSP-regret of \UCRLtwo applied to this problem.

\begin{theorem}\label{theorem_regret_SSP_UCRL_constant_costs_communicating}
For any policy $\pi\in\SD$, let $\rho_\pi := \lim_{T \rightarrow + \infty} \mathbb{E}_{\pi} \big[ \sum_{t=1}^{T} r_t/T \big]$ be the average reward of $\pi$ in the MDP $M_\infty$. Under Asm.\,\ref{assumption_uniform_costs}, we have
\begin{align*}
\pi^\star = \argmin_\pi V^\pi(s_0) = \argmin_\pi \mathbb{E}[\tau_\pi(s_0)] = \argmax_\pi \rho^\pi.
\end{align*}
With probability $1-\delta$, \UCRLtwo run for any $K \geq 1$ episodes suffers a regret
\begin{small}
\begin{align}\label{eq:regret.ucrl}
    \Delta(\UCRLtwo, K) \leq 34 (V^{\star}(s_0)\!+\! 1) D S \sqrt{A T_K \log\Big(\frac{T_K}{\delta}\Big)},
\end{align}
\end{small}
with
\begin{align}\label{eq:regret.ucrl.time}
T_K \leq 2 (V^{\star}(s_0)+1) K + \widetilde{O}\left( V^{\star}(s_0)^2 D^2 S^2 A \right). 
\end{align}
\vspace{-0.1in}
\end{theorem}
Up to logarithmic and lower-order terms, the previous bound scales as $\wt O(V^\star(s_0) D S \sqrt{AT_K})$. This can be contrasted with the infinite-horizon regret \mbox{$\Delta_{\infty} := T\rho^\star - \sum_t r_t$} of \UCRLtwo, which in general infinite-horizon problems scales as $\wt O(D_\infty S \sqrt{AT})$, where \mbox{$D_{\infty} := \max_{s \neq s' \in \mathcal{S}'} \min_{\pi \in \SD} \mathbb{E}\left[\tau_{\pi}(s \rightarrow s')\right]$} is the diameter of $M_\infty$~\citep{jaksch2010near} and measures the longest shortest path between \textit{any} two states. We first notice that the ``extra'' factor $V^\star(s_0)$ is a direct consequence of the different definition of regret in the two settings. In fact, we have $\Delta = (V^\star (s_0)+1)\Delta_{\infty}$. As \UCRLtwo is designed for general infinite-horizon problems, we can only bound the regret~$\Delta_{\infty}$ and use the previous equality to translate it into the corresponding SSP-regret. As such, the factor $V^\star(s_0)$ is the price to pay for adapting \UCRLtwo to the SSP case. 
On the other hand, it is easy to see that in general $D \leq D_{\infty}$. Interestingly, Asm.\,\ref{assumption_finite_SSP_diameter} does not imply that $M_\infty$ is communicating, which is needed for proving regret bounds for \UCRLtwo in general MDPs. Thm.\,\ref{theorem_regret_SSP_UCRL_constant_costs_communicating} shows that even when $M_\infty$ is weakly-communicating ($D_{\infty} = + \infty$) and some states may not be accessible from one another, \UCRLtwo is able to adapt to the SSP nature of the problem and achieve a bounded regret.

Importantly, notice that no assumption is made about the properness of the policies. The key for \UCRLtwo to manage policies that may never reach the goal state is the construction of \textit{internal} episodes, where policies are interrupted when the number of samples collected in a state-action pair is doubled. This allows \UCRLtwo to avoid accumulating too much regret when executing non-proper policies (they are eventually stopped) and, at the same time, perform well when the current policy is near-optimal (it is not stopped too early). Nonetheless, the stopping condition only relies on the number of samples and it is completely agnostic to the episodic nature of the SSP problem.

While the previous analysis suggests that algorithms for infinite-horizon MDPs could be readily executed in SSP problems with strong regret guarantees, this is no longer the case when moving to the general setting of non-uniform costs. Indeed, in order to estimate the performance of a stationary policy w.r.t.\,its value function, we cannot use the average-cost criterion since it does not capture the incentive to reach the goal state. 
As an illustrative example, consider the deterministic two-state SSP $M$ from Fig.\,\ref{fig:toy_SSP}. The optimal SSP policy $\pi^\star$ always selects action $a_2$ since it has minimal value $V^{\star}(s_0) = c_{\max}$. The optimal infinite-horizon policy always selects action $a_1$ since it has minimal average cost $\rho^\star = c_{\min}$, whereas $\rho_{\pi^\star} = c_{\max} / 2$. Consequently, running \UCRLtwo in general SSP may converge to a suboptimal policy and yield linear SSP-regret.

In the next section, we propose a novel algorithm designed to target the general SSP objective function (non-uniform costs) with a two-phase structure and a carefully designed condition to interrupt executing policies.

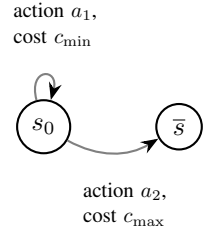
\begin{figure}[t]
\begin{minipage}[t]{0.49\columnwidth}
	\centering
	\begin{tikzpicture}[thick,scale=0.9]
	\begin{scope}[every node/.style={circle,thick,draw}]
	\node (1) at (0,0) {$s_0$};
	\node (2) at (2,0) {$\overline{s}$};
	\end{scope}
	\begin{scope}[>={Stealth[black]},
	every node/.style={fill=white,circle},
	every edge/.style={draw=gray, thick},
	every loop/.style={draw=gray, thick, min distance=5mm,looseness=5}]
	\path[]
	(1) [->,thick] edge[in=200,out=150, loop above,thick] node[above=0.05, scale=0.8, text width=10mm] {action~$a_{1}$, cost~$c_{\min}$} (1)
	(1) [->,thick] edge[bend right=30] node[below=0.05, scale=0.8, text width = 10mm] {action~$a_{2}$, cost~$c_{\max}$} (2);
	\end{scope}
	\end{tikzpicture}
\end{minipage}%
\begin{minipage}[t]{0.51\columnwidth}
\vspace{-1.25in}
\caption{Deterministic two-state SSP $M$ with two available actions: $a_1$ which self-loops on $s_0$ with cost $c_{\min}$ and $a_2$ which goes from $s_0$ to $\overline{s}$ with cost $c_{\max} > 2 c_{\min}$.}
\label{fig:toy_SSP}
\end{minipage}

\vspace{-0.2in}
\end{figure}




\section{General SSP}
\label{section_general_SSP}

The general SSP problem requires \textit{(i)} to quickly reach the goal state while \textit{(ii)} at the same time minimizing the cumulative costs. On the one hand, if we constrain the costs to be all equal, objectives \textit{(i)} and \textit{(ii)} coincide and the SSP problem can be addressed using infinite-horizon algorithms as seen in Sect.\,\ref{section_uniform_costs}. On the other hand, all previous works in the SSP setting constrain the hitting time of \textit{all} policies (i.e., the loop-free assumption), which means that objective \textit{(i)} is always guaranteed and the algorithm can focus its efforts on objective \textit{(ii)}.

In this section, we tackle head-on the general SSP problem for the first time, where we need \textit{to optimize over the two possibly conflicting objectives \textit{(i)} and \textit{(ii)} at the same time.} This poses algorithmic and technical challenges (e.g., non-proper policies may never reach the goal state and have unbounded value function) that require devising a novel optimistic algorithm, specifically designed for SSP problems.

\subsection{The \UCSSP Algorithm}
\label{sect_ucssp_alg}

\begin{algorithm}[t]
    \begin{small}
    \begin{algorithmic}
      \STATE \textbf{Input:} Confidence $\delta \in (0,1)$, costs, $\mathcal{S}', \mathcal{A}$.
      \STATE \textbf{Initialization:} Set the state-action counter $N_{0,0}(s,a) := 0$ for any $(s,a) \in \mathcal{S} \times \mathcal{A}$ and the time step $t := 1$. 
      \STATE Set $k := 0$. \CommentSty{\scriptsize{// episode index}}
      \STATE Set $G_{0,0} := 0$. \CommentSty{\scriptsize{// number of attempts in phases \ding{173}}} 
      \WHILE{$k < K$}
        \STATE \tcp{\scriptsize{New environmental episode}}
        \STATE Increment $k \mathrel{+}= 1$.
        \STATE Set $j := 0$ \CommentSty{\scriptsize{// attempts in phase \ding{173} of episode $k$}}
        \WHILE{$s_t \neq \overline{s}$}
            \STATE Set $t_{k,j} := t$ and counter $\nu_{k,j}(s,a) := 0$.
            \STATE Set $G_{k,j} = G_{k,0} + j$
            \STATE Compute $(\wt{\pi}_{k,j}, H_{k,j}) := \EVI_{\textrm{SSP}}(k,j)$.
            \WHILE{$t \leq t_{k,j} + H_{k,j}$ ~{\bf and}~ $s_t \neq \overline{s}$}
                \STATE Execute action $a_t =\wt{\pi}_{k,j}(s_t)$, observe cost $c(s_t,a_t)$ and next state $s_{t+1}$.
                \STATE Set $\nu_{k,j}(s_t,a_t) \mathrel{+}= 1$.
                \STATE Set $t \mathrel{+}= 1$.
            \ENDWHILE
            \IF{$s_t \neq \overline{s}$}
            \STATE \tcp{\scriptsize{Switch to phase \ding{173}}}
                \STATE Set $N_{k,j+1}(s,a) := N_{k,j}(s,a) + \nu_{k,j}(s,a)$
                \STATE Set $j \mathrel{+}= 1$
            \ENDIF
        \ENDWHILE
        \STATE Set $N_{k+1,0}(s,a) := N_{k,j}(s,a)+ \nu_{k,j}(s,a)$.
        \STATE Set $G_{k+1,0} := G_{k,j}$.
      \ENDWHILE
    \end{algorithmic}
    \end{small}
    \caption{\UCSSP algorithm}
    \label{algorithm_UCSSP}
\end{algorithm}

\begin{algorithm}[t]
    \caption{$\EVI_{\textrm{SSP}}$}
    \label{algorithm_VI_SSP}
    \begin{small}
        \begin{algorithmic}
            \STATE \textbf{Input:} Attempt index $(k,j)$ and $N_{k,j}(s,a)$ samples.
            \IF{$j=0$}
                \STATE $\varepsilon_{k,0} := \frac{c_{\min}}{2t_{k,0}}$, $\gamma_{k,j} := \frac{1}{\sqrt{k}}$.
            \ELSE
                \STATE $\varepsilon_{k,j} := \frac{1}{2t_{k,j}}$, $\gamma_{k,j} := \frac{1}{\sqrt{G_{k,j}}}$.
            \ENDIF
            \STATE Compute estimates $\widehat{p}_{k,j}$ and confidence set $\mathcal{M}_{k,j}$ with the $N_{k,j}$ samples collected so far.
            \STATE Define the extended optimal Bellman operator $\wt{\mathcal{L}}_{k,j}$ as in Eq.\,\eqref{eq_bellman_operator}.
            \vspace{-0.14in}
            \STATE \tcp{\small EVI scheme}
            \STATE Set $m:=0$, $v_0 := \textbf{0}$ ($S$-sized vector) and $v_1 := \wt{\mathcal{L}}_{k,j} v_0$.
            \WHILE{$\norm{v_{m+1} - v_m}_{\infty} > \varepsilon_{k,j}$}
                \STATE $m \mathrel{+}= 1$.
                \STATE $v_{m+1} := \wt{\mathcal{L}}_{k,j} v_m$.
            \ENDWHILE
            \STATE Set $\wt{v}_{k,j} := v_m$.
            \STATE Compute $\wt{\pi}_{k,j}$ the optimistic greedy policy w.r.t.\,$\wt{v}_{k,j}$.
            \STATE Compute $\wt{p}_{k,j}$ the corresponding optimistic model.
            \STATE Compute $\wt{Q}_{k,j}$ the transition matrix of $\wt{\pi}_{k,j}$ in the optimistic model $\wt{p}_{k,j}$ over $\mathcal{S}$, i.e., 
            for any $(s, s') \in \mathcal{S}^2$,
                \begin{align*}
                    \wt{Q}_{k,j}(s,s') := \sum_{a \in \mathcal{A}} \wt{\pi}_{k,j}(a \vert s) \wt{p}_{k,j}(s' \vert s,a).
                \end{align*}
            \vspace{-0.08in}
            \STATE Compute $H_{k,j} := \min \left\{ n > 1 : \norm{\wt{Q}_{k,j}^{n-1}}_{\infty} \leq \gamma_{k,j} \right\}$.
            \STATE \textbf{Output:} policy $\wt{\pi}_{k,j}$ and horizon $H_{k,j}$.
        \end{algorithmic}
    \end{small}
\end{algorithm}

We present \UCSSP, an algorithm for efficient exploration in general SSP problems (Alg.\,\ref{algorithm_UCSSP}). At a high level, \UCSSP proceeds through each environmental episode $k$ in a \textit{two-phase} fashion. In phase \ding{172}, \UCSSP executes a policy trying to solve the SSP problem by tackling both objectives~\textit{(i)} and~\textit{(ii)} (i.e., reach the goal while minimizing the cumulative costs). We refer to this first policy as an \textit{attempt} in phase~\ding{172}. As \UCSSP relies on estimates of the true (unknown) SSP, it may select a non-proper policy that would never reach the goal state and incur an unbounded regret. In order to avoid this situation, if the goal state is not reached after a given \textit{pivot} horizon, the algorithm deems the whole episode as a \textit{failure} and it switches to phase \ding{173}, whose only objective is to terminate the episode as fast as possible (i.e., it only considers objective \textit{(i)} and disregards the costs). Nonetheless, optimizing an estimate of the hitting time (i.e., objective \textit{(i)}) does not guarantee that the corresponding policy successfully reaches the goal state (i.e., is proper) and multiple \textit{attempts} (i.e., policies) in phase \ding{173} may be needed. Similar to phase \ding{172}, whenever the goal state is not reached after a certain \textit{pivot} horizon, the current policy is terminated and a new policy is computed. Phase \ding{173} and the overall episode ends when the goal state is eventually reached. Notation-wise, the $k$-th phase \ding{172} is indexed by $(k,0)$ (note that $k$ coincides with the current number of episodes), while the $j$-th attempt in the phase \ding{173} of episode $k$ is indexed by $(k,j)$ for $j \geq 1$. Moreover, we denote by $J_k$ the number of attempts performed during the phase \ding{173} of episode $k$, and by $G_{k,j}$ the total number of attempts in phases \ding{173} up to (and including) attempt $(k,j)$. 

\textbf{Optimistic policies.} 
\UCSSP relies on the principle of \textit{optimism in face of uncertainty}. At each attempt, it executes a policy with either lowest optimistic (cost-weighted) value for an attempt in phase \ding{172}, or with lowest optimistic expected hitting time for an attempt in phase \ding{173}. At the beginning of any attempt $(k,j)$, the algorithm computes a set of plausible MDPs defined as $\mathcal{M}_{k,j} := \{ \langle \mathcal{S}, \mathcal{A}, c, \widetilde{p} \rangle ~\vert ~ \widetilde{p}(\cdot|s,a) \in B_{k,j}(s,a) \}$ where $B_{k,j}(s,a)$ is a high-probability confidence set on the transition probabilities of the true MDP $M$. We set $B_{k,j}(s,a) := \{ \widetilde{p} \in \mathcal{C} ~\vert ~ \widetilde{p}(\cdot \, \vert \,  \overline{s},a) = \mathds{1}_{\overline{s}}, \norm{\widetilde{p}(\cdot \, \vert \,  s,a) - \widehat{p}_{k,j}(\cdot \, \vert \,  s,a)}_1 \leq \beta_{k,j}(s,a)\}$, with $\mathcal{C}$ the $S'$-dimensional simplex, $\widehat{p}_{k,j}$ the empirical average of transitions prior to attempt $(k,j)$ and
\begin{align*}
\beta_{k,j}(s,a) := \sqrt{\frac{8S\log\big(2AN_{k,j}^+(s,a)\,\delta^{-1}\big)}{N_{k,j}^+(s,a)}},
\end{align*}
where $N_{k,j}^{+}(s,a) := \max \{1,N_{k,j}(s,a) \} $ with $N_{k,j}$ being the state-action counts prior to attempt $(k,j)$. 
The construction of $\beta_{k,j}(s,a)$ 
guarantees that $M \in \mathcal{M}_{k,j}$ with high probability, as shown in the following lemma.

\begin{lemma}
Introduce the event $\mathcal{E} := \bigcap_{k=1}^{+ \infty} \bigcap_{j=1}^{J_k} \{ M \in \mathcal{M}_{k,j} \}$. Then $\mathbb{P}(\mathcal{E}) \geq 1 - \frac{\delta}{3}$.
\label{lemma_high_probability_intersection_bound_M_in_M_k}
\end{lemma}

Once $\mathcal{M}_{k,j}$ has been computed, \UCSSP applies an extended value iteration (\EVI) scheme (Alg.\,\ref{algorithm_VI_SSP}) to compute a policy with lowest optimistic value (if $j=0$) or lowest optimistic expected hitting time (if $j \geq 1$). Formally, we define the extended optimal Bellman operator $\widetilde{\mathcal{L}}_{k,j}$ such that for any $v\in \mathbb{R}^{\mathcal{S}}$ and $s\in \mathcal{S}$,
\begin{align}
        \widetilde{\mathcal{L}}_{k,j}v(s) := &\min_{a \in \mathcal{A}} \Big\{ c_{k,j}(s,a) \nonumber \\ &+ \min_{\widetilde{p}\in B_{k,j}(s,a)}\sum_{y \in \mathcal{S}} \widetilde{p}(y \, \vert \,  s,a) v(y) \Big\},
\label{eq_bellman_operator}
\end{align}
where the costs $c_{k,j}$ depend on the phase as follows
\begin{align*}
    c_{k,j}(s,a) := \left\{
    \begin{array}{ll}
        c(s,a) & \mbox{if~} j=0 \\
        1 & \mbox{otherwise.}
    \end{array}
\right. 
\end{align*}

As explained by \citet[][Sect.\,3.1]{jaksch2010near}, we can combine all the MDPs in $\mathcal{M}_{k,j}$ into a single MDP $\widetilde{M}$ with extended action set $\mathcal{A}'$. As proved by \citet[][Sect.\,3.3]{bertsekas1995dynamic} about the generalization of the SSP results to a compact action set, the Bellman operator $\widetilde{\mathcal{L}}_{k,j}$ satisfies the contraction property and thus $\EVI_{\textrm{SSP}}$ converges to a vector we denote by $\widetilde{V}_{k,j}^{\star}$. We have the following component-wise inequalities when the stopping condition of Alg.\,\ref{algorithm_VI_SSP} is met.\footnote{Note that the stopping condition is different from the standard one for \VI for average reward MDPs~\citep[see e.g.,][]{puterman2014markov, jaksch2010near} that is defined in span seminorm. Also note that as opposed to standard \VI, we do not have guarantees of the type $\|v_n - \widetilde{V}_{k,j}^{\star}\|_{\infty} \leq \epsilon$ where $\widetilde{V}_{k,j}^{\star} = \widetilde{\mathcal{L}}_{k,j} \widetilde{V}_{k,j}^{\star}$.}

\begin{lemma}
    For any attempt $(k,j)$, denote by $\widetilde{v}_{k,j}$ the output of $\EVI_{\textrm{SSP}}$ with operator $\widetilde{\mathcal{L}}_{k,j}$ and accuracy $\varepsilon_{k,j}$. Then $\widetilde{\mathcal{L}}_{k,j} \widetilde{v}_{k,j} \leq \widetilde{v}_{k,j} + \varepsilon_{k,j}$. Furthermore, under the event $\mathcal{E}$ we have $\widetilde{v}_{k,j} \leq V^{\star}$ if $j=0$ or $\widetilde{v}_{k,j} \leq \min_{\pi} \mathbb{E}(\tau_{\pi})$ otherwise.
\label{lemma_optimism}
\end{lemma}

The optimistic policy $\wt{\pi}_{k,j}$ executed during attempt $(k,j)$ is the greedy policy w.r.t.\,$\wt{v}_{k,j}$. We also denote by $\wt{p}_{k,j}$ the optimistic transition probabilities and by $\wt{Q}_{k,j}$ the transition matrix of $\wt{\pi}_{k,j}$ in $\wt{p}_{k,j}$ over the non-goal states $\mathcal{S}$.

\textbf{The pivot horizon.}
A crucial aspect for the correct functioning of the algorithm is to carefully select the ``pivot'' horizon. If the pivot horizon is too small, the algorithm may switch from phase \ding{172} to \ding{173} too quickly and may perform too many attempts in phase \ding{173}. As the policies in phase \ding{173} completely disregard the costs, they may lead to suffer large regret. On the other hand, if the pivot horizon is too large and \UCSSP selects a non-proper policy in phase \ding{172}, then the regret accumulated during phase \ding{172} would be too large.

We select the following length for attempt $(k,j)$
\begin{small}
\begin{equation}
\hspace{-0.015in}~H_{k,j} = \min \Big\{ n \!>\! 1: \norm{(\wt{Q}_{k,j})^{n-1}}_{\infty} \leq \frac{\mathds{1}_{j=0}}{\sqrt{k}} +\! \frac{\mathds{1}_{j \geq 1}}{\sqrt{G_{k,j}}} \Big\}.
\label{eq_choice_H_k}
\end{equation}
\end{small}%
If $\wt{\pi}_{k,j}$ is executed for $H_{k,j}$ steps without reaching $\overline{s}$, then attempt $(k,j)$ is said to have \textit{failed} and the next attempt $(k,j+1)$ (necessarily in phase \ding{173}) is performed. Otherwise, the attempt is said to have \textit{succeeded}, a new episode begins and the next attempt $(k+1,0)$ (in phase \ding{172}) is performed.

Denote by $\wt{\tau}_{k,j}$ the hitting time in the model $\wt{p}_{k,j}$ of the policy $\wt{\pi}_{k,j}$. We first prove that $\wt{\pi}_{k,j}$ is proper in $\wt{p}_{k,j}$ by connecting its value function to $\wt{v}_{k,j}$, which is finite from Lem.\,\ref{lemma_optimism} (see App.\,\ref{app:bound.maximal.episode.length} and Eq.\,\ref{cost_weighted_optimism_expected_hitting_time}). As a result, $\wt{\tau}_{k,j}$ follows a \textit{discrete PH distribution} and plugging Prop.\,\ref{equation_discrete_phase_type_distribution} into Eq.\,\eqref{eq_choice_H_k} entails that
\begin{align*}
 \max_{s \in \mathcal{S}} \mathbb{P}(\wt{\tau}_{k,j}(s) \geq H_{k,j}) \leq \frac{\mathds{1}_{j=0}}{\sqrt{k}} + \frac{\mathds{1}_{j \geq 1}}{\sqrt{G_{k,j}}}.
\end{align*}
$H_{k,j}$ is thus selected so that the tail probability of the \textit{optimistic} hitting time is small enough, i.e., there is a high probability that $\wt{\pi}_{k,j}$ will \textit{optimistically} reach $\overline{s}$ within $H_{k,j}$ steps. The maximum over $s \in \mathcal{S}$ guarantees this property for any state $s$ from which attempt $(k,j)$ begins (since attempts in phase \ding{173} do not necessarily start at $s_0$).

\subsection{Regret Analysis of \UCSSP}
\label{subsection_general_env}

As proved in the following theorem, \UCSSP is the first no-regret learning algorithm in the general SSP setting. 

\begin{theorem} With overwhelming probability, for any $K \geq 1$, if at each attempt $(k,j)$ $\EVI_{\textrm{SSP}}$ is run with accuracy $\varepsilon_{k,j} := \frac{c_{\min} \mathds{1}_{j =0} + \mathds{1}_{j \geq 1}}{2t_{k,j}}$, where $t_{k,j}$ is the time index at the beginning of the attempt, then \UCSSP suffers a regret
\begin{align*}
\Delta(\UCSSP, K) = \wt{O}\Big( &c_{\max} D S  \sqrt{\frac{ c_{\max}}{c_{\min}} A D K} \\ &+ c_{\max} S^2 A D^2 \Big).
\end{align*}


\label{theorem_regret_nonproper}
\end{theorem}

\textbf{Dependency on $K$ and $D$.} \jt{Significantly, \UCSSP achieves an overall rate $\wt O(\sqrt{K})$ which is optimal w.r.t.\,the number of episodes $K$.} The bound also illustrates how \UCSSP is able to adapt to the complexity of navigating through the MDP as shown by the dependency on the SSP-diameter $D$, which measures the longest shortest path to the goal state from any state. Interestingly, this is achieved without any prior knowledge either on an upper bound of the optimal value function $V^\star$ (or of the SSP-diameter itself), or whether the set of policies $\SD$ contains proper policies or not. We can further inspect the dependency on $D$ by rewriting the regret bound of \UCSSP, which scales as $D^{3/2} \sqrt{K}$ in Thm.\,\ref{theorem_regret_nonproper}, as $D \sqrt{T_K}$, where $T_K$ is the total number of steps executed until the end of episode of $K$.\footnote{Even though $T_K$ is a \textit{random} quantity, inspecting the proof (see Sect.\,\ref{proof_sketch}) provides a bound $T_K\lesssim DK$ for $K$ large enough.} As shown in Lem.\,\ref{lemma_1}, up to a factor of $c_{\max}$, the SSP-diameter $D$ is an upper bound on the range of the optimal value function and as such it can be (qualitatively) related to the horizon $H$ in the finite-horizon setting and the diameter $D_{\infty}$ in the infinite-horizon setting, which bound the range of the optimal value function and bias function respectively. 



\textbf{Dependency on cost range.} The multiplicative constant $\frac{c_{\max}}{c_{\min}}$ appearing in the bound quantifies the range of the cost function and accounts for the difference from the uniform-cost setting. 
Interestingly, the presence of the ratio $\frac{c_{\max}}{c_{\min}}$ implies that the regret bound is not invariant w.r.t.\,a uniform additive perturbation of all costs. This behavior, which does not appear in the finite- or infinite-horizon settings, stems from the fact that an additive offset of costs may alter the optimal policy in the SSP sense (see Lem.\,\ref{lemma_costs_offset}, App.\,\ref{app_relax_asm}).

While the previous discussion shows that \UCSSP successfully tackles general SSP problems, we can also study its behavior in the limit (and much simpler) cases of uniform-cost and loop-free SSP, and compare its regret to infinite- and finite-horizon algorithms respectively.

\textbf{Uniform-cost SSP.} Under Asm.\,\ref{assumption_uniform_costs}, \UCSSP achieves a regret of $\wt O(DS\sqrt{ADK})$, in contrast with the bound $\wt O(V^\star(s_0) DS \sqrt{AV^\star(s_0)K})$ of \UCRLtwo derived in Sect.\,\ref{section_uniform_costs}. While in this restricted setting \UCRLtwo performs better when $s_0$ is a privileged starting state to reach $\overline{s}$ compared to the rest of states in $\mathcal{S}$, \UCSSP yields an improvement over \UCRLtwo whenever $V^\star(s_0) \geq D^{1/3}$. \jt{Our experiments in App.\,\ref{app:experiments} illustrate that \UCSSP suffers smaller regret than \UCRLtwo in a gridworld with uniform costs, showcasing that \UCSSP manages to better adapt to the goal-oriented structure of the problem.}


\textbf{Loop-free SSP.} Let us assume that there exists a \textit{known} upper bound $H$ on the hitting time of \textit{any} policy. Then a slight variation of the finite-horizon algorithm \UCBVI \citep{azar2017minimax} can be applied. While its bound would scale as $\wt{O}(\sqrt{HSAT})$ and showcase an improved $\sqrt{S}$-dependency, it would regrettably scale with $\sqrt{H}$ which may be much larger than the $D$ factor appearing in Thm.\,\ref{theorem_regret_nonproper} as soon as the hitting times $\tau_{\pi}$ differ significantly across policies $\pi$. Moreover, \UCSSP does not require the prior knowledge of $H$, as opposed to \UCBVI or any other existing algorithm in the finite-horizon or loop-free setting.



The analysis of \UCSSP reveals the crucial role of the pivot horizon in shaping the behavior and performance of the algorithm. In the uniform-cost case, $\EVI_{\textrm{SSP}}$ and standard \EVI~used in \UCRLtwo both converge to the same policy. The main difference between the two algorithms consists in the stopping criterion for the execution of the optimistic policy. While \UCRLtwo applies a generic doubling scheme (i.e., an internal episode is terminated when the number of samples is doubled in at least a state-action pair), \UCSSP leverages the episodic nature of the SSP problem and sets a pivot horizon such that the current policy should successfully terminate with high (optimistic) probability. In the loop-free setting, \UCBVI picks a single policy per episode and waits until termination. While all policies are guaranteed to terminate in finite time, the length of the episode may still be very long. On the other hand, \UCSSP goes through different policies within each episode whenever they are taking \textit{too long} to reach the goal state.

\subsection{Proof Sketch of Thm.\,\ref{theorem_regret_nonproper}}
\label{proof_sketch}

As explained in Sect.\,\ref{section_SSP}, tackling the general SSP problem requires introducing the novel notion of SSP-regret. It can neither be managed by a step-by-step comparison between the algorithmic and
optimal performances as in infinite-horizon, nor by an episode-by-episode comparison as in finite-horizon. We thus need to derive a new analysis to handle the specificities of the SSP-regret. 

Denoting by $T_K$ the total number of steps at the end of episode $K$, we decompose $T_K = T_{K,1} + T_{K,2}$, with $T_{K,1}$ (resp.\,$T_{K,2}$) the total time during attempts in phase \ding{172} (resp.\,phase \ding{173}). We introduce the \textit{truncated} regret
\begin{align}
    \mathcal{W}_K := \sum_{k=1}^K \bigg[ \Big( \sum_{h=1}^{H_{k,0}} c(s_{k,h}, \wt{\pi}_{k,0}(s_{k,h})) \Big) - V^{\star}(s_{0}) \bigg],
\label{eq_truncated_regret}
\end{align}
which is obtained by considering the cumulative cost up to $H_{k,0}$ steps rather than for the actual duration of each attempt in phase \ding{172}. By assigning a regret of $c_{\max}$ to each step in phase \ding{173}, we can then decompose the regret as
\begin{align}
    \Delta(\UCSSP,K) \leq \mathcal{W}_K + c_{\max} T_{K,2}.
\label{regret_decomposition}
\end{align}
This decomposition directly justifies the different nature of the two phases employed by \UCSSP. While phase \ding{172} directly tries to minimize $\mathcal{W}_K$, phase \ding{173} only needs to keep $T_{K,2}$ under control, which requires executing policies that reach the goal state as quickly as possible.

\paragraph{Bound on $\mathcal{W}_K$.} 

We first bound $\mathcal{W}_K$ by drawing inspiration from techniques in the finite-horizon setting (see e.g., \citealp{azar2017minimax}), by successively unrolling the Bellman operator to get a telescopic sum which can be bounded using the Azuma-Hoeffding inequality and a pigeonhole principle.
\begin{lemma}\label{regret_key_lemma}
Introduce $\Omega_{K} := \max_{k \in [K]} H_{k,0}$. With probability at least $1-\delta$,
\begin{align*}
    \mathcal{W}_K &= O \bigg( c_{\max} D S \sqrt{A \Omega_K K \log\Big(\frac{ \Omega_K K}{\delta}\Big)} \bigg). 
\end{align*}
\end{lemma}

\paragraph{Bound on $\Omega_K$.}

On the one hand, since $\mathcal{W}_K$ directly scales with $\sqrt{\Omega_K}$, we must ensure that the lengths of attempts in phase \ding{172} are not too long. Ideally, we would set them as relatively tight upper bounds of $V^{\star}(s_0)$ or $D$, yet these are critically \textit{unknown}. Instead, in Eq.\,\eqref{eq_choice_H_k} we tune the lengths $H_{k,0}$ depending on optimistic quantities (which can be easily computed at the start of each attempt), and prove in the following lemma that they crucially scale as $\wt{O}(D)$.
\begin{lemma}
Under the event $\mathcal{E}$,
\begin{align*} 
    \Omega_{K} \leq \left\lceil 6 \frac{c_{\max}}{c_{\min}} D \log(2 \sqrt{K}) \right\rceil.
\end{align*}
\label{lemma_bound_H_k_analysis_1}
\end{lemma}

\begin{proof}[Proof sketch]
Consider a state $y \in \mathcal{S}$ such that
\begin{align*}
\norm{(\wt{Q}_{k,0})^{H_{k,0}-2}}_{\infty} = \mathds{1}_y^\top (\wt{Q}_{k,0})^{H_{k,0}-2} \mathds{1}.
\end{align*}
From Lem.\,\ref{equation_discrete_phase_type_distribution}, the above is equal to $\mathbb{P}(\wt{\tau}_{k,0}(y) \geq H_{k,0} - 1)$. To bound it, we apply a corollary of Markov's inequality 
\begin{align*}
    \mathbb{P}(\wt{\tau}_{k,0}(y) \geq H_{k,0} - 1) \leq \frac{\mathbb{E}\left[(\wt{\tau}_{k,0})^r \right]}{(H_{k,0}-1)^r},
\end{align*}
for a carefully chosen exponent $r := \lceil \log(2 \sqrt{k}) \rceil \geq 1$. We then prove that $\wt{\tau}_{k,0}$ follows a discrete PH distribution that satisfies $\mathbb{E}\left[\wt{\tau}_{k,0}(s) \right] \leq \frac{2c_{\max} D}{c_{\min}}$ for all $s \in \mathcal{S}$. This leads us to derive an upper bound on the $r$-th moment of any hitting time distribution with bounded expectation starting from any state (Lem.\,\ref{lem:raw_moments_discrete_ph_distribution}, App.\,\ref{app:bound.maximal.episode.length}, which may be of independent interest). 
Applying this result to $\wt{\tau}_{k,0}$ yields
\begin{align*}
    \mathbb{E}\left[(\wt{\tau}_{k,0})^r \right] \leq 2 \left( r \frac{2c_{\max}D}{c_{\min}} \right)^{r},
\end{align*}
which gives on the one hand
\vspace{-0.1in}
\begin{align*}
    \norm{(\wt{Q}_{k,0})^{H_{k,0}-2}}_{\infty} \leq \frac{ 2 \left( r \frac{2c_{\max}D}{c_{\min}} \right)^{r} }{(H_{k,0} - 1)^r}. 
\end{align*}
On the other hand, the choice of $H_{k,0}$ in Eq.\,\eqref{eq_choice_H_k} entails that
\begin{align*}
    \frac{1}{\sqrt{k}} < \norm{(\wt{Q}_{k,0})^{H_{k,0}-2}}_{\infty}.
\end{align*}
Combining the two previous inequalities finally provides the desired upper bound on $H_{k,0}$.
\end{proof}

\paragraph{Bound on $T_{K,2}$.}

On the other hand, since $T_{K,2}$ increases with the number of attempts in phase \ding{173}, we must ensure that there are not too many of such attempts and that their lengths can be adequately controlled. In light of this and leveraging the way the length $H_{k,0}$ is constructed (Eq.\,\ref{eq_choice_H_k}), we bound the number of failed attempts in phase \ding{172} up to episode $K$, which we denote by $F_K$. 

\begin{lemma}\label{lemma_bound_F}
With probability at least $1-\delta$, 
\begin{align*}
F_K &\leq 2 \sqrt{K} + 2 \sqrt{2 \Omega_{K} K \log\left(\frac{2 ( \Omega_{K} K)^2}{\delta}\right)} \\ &+ 4 S \sqrt{8 A \Omega_{K} K \log\left(\frac{2 A \Omega_{K} K}{\delta}\right)}.
\end{align*}
\end{lemma}

\begin{proof}[Proof sketch]
We write $F_K= F'_K + F''_K$ with $F'_K := \sum_{k=1}^K \mathbb{P}(\wt{\tau}_{k,0}(s_{0}) > H_{k,0})$ and $F''_K := \sum_{k=1}^K \left[ \mathds{1}_{\{ \tau_{k,0}(s_0) > H_{k,0} \}} - \mathbb{P}(\wt{\tau}_{k,0}(s_{0}) > H_{k,0}) \right]$. A martingale argument and the pigeonhole principle bound $F''_K$, while the choice of $H_{k,0}$ controls each summand of $F'_K$.
\end{proof}

\vspace{-0.01in}

Equipped with Lem.\,\ref{lemma_bound_F}, we proceed in bounding the total duration of the attempts in phase \ding{173}. 

\begin{lemma}\label{lemma_bound_duration_second_phases}
With probability at least $1-\delta$,
\vspace{-0.015in}
\begin{align*}
    T_{K,2} &= \wt{O}\left( D S  \sqrt{\frac{ c_{\max}}{c_{\min}} A D K} + S^2 A D^2 \right).
\end{align*}
\end{lemma}

\vspace{-0.05in}


Putting everything together, we obtain Thm.\,\ref{theorem_regret_nonproper} by plugging Lem.\,\ref{regret_key_lemma}, \ref{lemma_bound_H_k_analysis_1} and \ref{lemma_bound_duration_second_phases} into Eq.\,\eqref{regret_decomposition}. \jt{Note that while the regret decomposition in the two-phase process (Eq.\,\ref{regret_decomposition}) has the advantage of making the analysis intuitive and modular, it renders Bernstein techniques less effective in capturing low-variance deviations, as opposed to the analysis of \UCBVI and \UCRLB \citep{improved_analysis_UCRL2B} which shave off a term of $\sqrt{H}$ or $\sqrt{D_{\infty}}$ for large enough time steps in the finite- and infinite-horizon settings, respectively.}

\vspace{-0.03in}

\section{Relaxation of Assumptions}
\label{section_relaxation_asm}

Although Asm.\,\ref{assumption_costs} and \ref{assumption_finite_SSP_diameter} seem natural in the SSP problem, we design variants of \UCSSP that can handle dead-end states and/or zero costs. We defer to App.\,\ref{app_relax_asm} the complete analysis.

\textbf{Relaxation of Asm.\,\ref{assumption_finite_SSP_diameter} ($D = +\infty$).}
If $M$ is non-SSP-communicating, there exists at least one (possibly unknown) \textit{dead-end} state from which reaching the goal $\overline{s}$ is impossible. This implies that $\EVI_{\SSP}$, which operates on the entire state space $\mathcal{S}$, fails to converge since the values at dead-end states are infinite. To tackle this problem, we assume that the agent has prior knowledge on an upper bound $J \geq V^{\star}(s_0)$ and that it has at any time step the ``resetting'' ability to transition with probability 1 to $s_0$ with a cost of $J$ (to prevent it from getting stuck). Equipped with these two assumptions, by optimizing a value function that is \textit{truncated} at $J$ \citep{kolobov2012theory}, we prove that a variant of \UCSSP achieves a regret guarantee identical to Thm.\,\ref{theorem_regret_nonproper} except that the infinite term $D$ is replaced by $J$ (see Lem.\,\ref{regret_J_asm_relaxed}, App.\,\ref{subsection_relaxation_asm_finite_SSP_diameter}). 

\textbf{Relaxation of Asm.\,\ref{assumption_costs} ($c_{\min}= 0$).}
\jt{Under the existence of zero costs, the optimal policy is not even guaranteed to be proper \cite{bertsekas1995dynamic}. We thus change the definition of SSP-regret and compare to the best \textit{proper} policy, by considering as optimal comparator the quantity $\min_{\pi \in \PSD} V^\pi$ instead of $\min_{\pi \in \SD} V^\pi$.} We observe that having $c_{\min}= 0$ renders the bound on $\Omega_K$ of Lem.\,\ref{lemma_bound_H_k_analysis_1} vacuous. To circumvent this issue, we introduce an additive perturbation $\eta_{k,0} > 0$ to the cost of each transition in the \textit{optimistic model} of each attempt $(k,0)$. Our resulting variant of \UCSSP achieves a $\wt{O}(K^{2/3})$ regret bound (see Lem.\,\ref{regret_cmin_0_asm_relaxed}, App.\,\ref{subsection_relaxation_asm_positive_costs} for the complete bound). The difference in rate ($K^{2/3}$ vs.\,$\sqrt{K}$) compared to Thm.\,\ref{theorem_regret_nonproper} stems from the fact that our procedure of offsetting the costs introduces a bias, which we minimize with the choice of perturbation $\eta_{k,0} = 1 / k^{1/3}$. \jt{Note that the later work of \citep{cohen2020near} devises an algorithm with a Bernstein-based analysis that achieves a $\sqrt{K}$-rate in the case $c_{\min} = 0$. }

\vspace{-0.25in}

\section{Conclusion and Extensions}

Although it encompasses numerous goal-oriented RL problems, the setting of episodic RL under its general SSP formulation had until now been neglected by the theoretical literature of RL, or had been studied under the strong, loop-free restriction on the MDP structure. Our key contribution is the design and analysis of \UCSSP, the first no-regret algorithm in the challenging setting of goal-oriented RL. Our analysis carefully combines existing techniques from the related settings of finite-horizon and infinite-horizon RL, as well as introduces refined ingredients to address the novel trade-off between minimizing costs and reaching the goal state. \jt{Interesting directions for further investigation include \textit{(1)} designing a model-free algorithm for exploration in SSP, and \textit{(2)} tackling SSP in the setting of linear function approximation.}



\bibliography{bibliography.bib}
\bibliographystyle{icml2020}

\newpage{}
\onecolumn
\appendix

\newpage

\section{Proof of Lem.\,\ref{lemma_wellposedproblem}, \ref{lemma_1} and~\ref{lemma_optimism}}
\label{app:proof_lemma_1}

\begin{proof}[Proof of Lem.\,\ref{lemma_wellposedproblem}]
	Asm.\,\ref{assumption_finite_SSP_diameter} implies that there exists at least one proper policy (i.e., $\PSD \neq \emptyset$), and Asm.\,\ref{assumption_costs} implies that for every non-proper policy $\pi$, the corresponding value function $V^{\pi}(s)$ is $+\infty$ for at least one state $s \in \mathcal{S}$. The rest follows from~\citet[][Sect.\,3.2]{bertsekas1995dynamic}.
\end{proof}

\begin{proof}[Proof of Lem.\,\ref{lemma_1}]
From the definition of the infinity norm and Asm.\,\ref{assumption_costs} and \ref{assumption_finite_SSP_diameter}, we have
\begin{align*}
    \norm{V^{\star}}_{\infty} = \max_{s \in \mathcal{S}} \min_{\pi \in \SD} \mathbb{E}\left[ \sum_{t = 1}^{\tau_{\pi}(s)} c(s_{t}, \pi(s_t)) \,\Big\vert\,s\right] \leq c_{\max} \max_{s \in \mathcal{S}} \min_{\pi \in \SD} \mathbb{E}\left[ \tau_{\pi}(s)\right] = c_{\max} D.
\end{align*}
\end{proof}

\begin{proof}[Proof of Lem.\,\ref{lemma_optimism}]
	The first inequality comes from the chosen stopping condition. As for the second, since we consider the initial vector $v^{(0)} = 0$, we know that $v^{(0)} \leq \widetilde{V}_{k,j}^{\star}$ with $\widetilde{V}_{k,j}^{\star} = \widetilde{\mathcal{L}}_{k,j}\widetilde{V}_{k,j}^{\star}$. By monotonicity of the operator $\widetilde{\mathcal{L}}_{k,j}$~\citep{puterman2014markov,bertsekas1995dynamic} we obtain $\widetilde{v}_{k,j} \leq \widetilde{V}_{k,j}^{\star}$. If $M \in \mathcal{M}_{k,j}$ and $j = 0$, then $\widetilde{V}_{k,j}^{\star} \leq V^{\star}$. If $M \in \mathcal{M}_{k,j}$ and $j \geq 1$, then all costs are equal to $1$ so the optimal value function is $\min_{\pi} \mathbb{E}(\tau_{\pi})$ and hence $\widetilde{V}_{k,j}^{\star} \leq \min_{\pi} \mathbb{E}(\tau_{\pi})$.
\end{proof}

\section{Proof of Thm.\,\ref{theorem_regret_SSP_UCRL_constant_costs_communicating}}
\label{app:uniform_costs}

Recall that we introduce the MDP $M_{\infty} := \langle \mathcal{S}', \mathcal{A}, r_{\infty}, p_{\infty}, s_0 \rangle$, with reward $r_{\infty} = \mathds{1}_{\overline{s}}$ and $p_{\infty}(\cdot \, \vert \,  s,a) = p(\cdot \, \vert \,  s,a)$ for $s \neq \overline{s}$ and $p_{\infty}(\cdot \, \vert \,  \overline{s},a) = \mathds{1}_{s_0}$ for all $a$. The SSP problem with uniform costs boils down to minimizing the expected hitting time of the goal state, which according to the following lemma is equivalent to maximizing the long-term average reward (or gain) in $M_{\infty}$. Recall that for any policy $\pi \in \SD$, its gain $\rho_{\pi}(s)$ starting from any $s \in \mathcal{S}$ is defined as
\begin{align*}
        \rho_{\pi}(s) := \lim_{T \rightarrow + \infty} \mathbb{E}_{\pi} \left[ \frac{1}{T} \sum_{t=1}^{T} r_{\infty}(s_t, \pi(s_t)) \,\Big\vert\,s \right].
\end{align*}

\begin{lemma}
Let $\pi_{\infty} \in \argmax_{\pi} \rho_{\pi}(s)$. Then $\pi_{\infty}$ is optimal in the SSP sense and its constant gain $\rho_{\infty}$ verifies
\begin{align*}
\rho_{\infty} = \frac{1}{V^{\star}(s_0) + 1}.
\end{align*}
\label{lemma_equivalence_SSP_UCRL_constant_costs}
\end{lemma}

\begin{proof}
Let $\pi$ be a policy such that $\overline{s}$ is reachable from $s_0$. Denote by $\mathcal{S}_{\pi}$ the set of communicating states for policy $\pi$ in $M_{\infty}$. Then the underlying Markov chain (restricted to $\mathcal{S}_{\pi}$) is irreducible with a finite number of states and is thus recurrent positive (see e.g.,~\citealp[][Thm.\,3.3]{bremaud2013markov}). Denoting by $\mu_{\pi}$ its unique stationary distribution, we have almost surely that
\begin{align*}
        \rho_{\pi}(s) = \lim_{T \rightarrow + \infty} \mathbb{E}_{\pi} \left[ \frac{\sum_{t=1}^{T} r_t}{T} \right] = \lim_{T \rightarrow + \infty} \mathbb{E}_{\pi} \left[ \frac{\sum_{t=1}^{T} \mathds{1}_{ \{ s_t = \overline{s} \}}}{T} \right] \myeqa \sum_{s \in \mathcal{S}_{\pi}} \mathds{1}_{ \{ s = \overline{s} \} } \mu_{\pi}(s) \myeqb \frac{1}{1 + \mathbb{E}\left[\tau_{\pi}(s_0)\right]},
\end{align*}
where (a) comes from the Ergodic Theorem for Markov Chains (see e.g.,~\citealp[][Thm.\,4.1]{bremaud2013markov}) and (b) uses the fact that $1/\mu_{\pi}(\overline{s})$ corresponds to the mean return time in state $\overline{s}$, i.e., the expected time to reach $\overline{s}$ starting from $\overline{s}$. We conclude with the fact that $V^{\pi}(s_0) = \mathbb{E}\left[\tau_{\pi}(s_0)\right]$.
\end{proof}

Hence, we can prove that \UCRLtwo satisfies the following SSP-regret bound.

\begin{lemma}
Under Asm.\,\ref{assumption_uniform_costs}, with probability at least $1-\delta$, for any $K \geq 1$,
\begin{align*}
    \Delta(\UCRLtwo, K) \leq 34 \left(V^{\star}(s_0)+1\right) D S \sqrt{A T_K \log\left(\frac{T_K}{\delta}\right)},
\end{align*}
where $T_K = \sum_{k=1}^K \tau_k(s_0)$.
\label{lemma_regret_SSP_UCRL_constant_costs_communicating}
\end{lemma}

\begin{proof}
Using the fact that $K = \sum_{t=1}^{T_K} \mathds{1}_{ \{s_t = \overline{s} \}}$, the SSP-regret can be written as
\begin{align*}
\Delta(\mathfrak{A}, K) = \sum_{k = 1}^{K} \left[ \sum_{t=1}^{\tau_k} \mathds{1}_{ \{s_t \neq \overline{s} \}} - V^{\star}(s_{0}) \right] = T_K - K - V^{\star}(s_0) K = T_K - (V^{\star}(s_0) + 1)K.
\end{align*}
For any $T \geq 1$ denote by $\Delta_{\infty}(\mathfrak{A}, T, M_{\infty})$ the (reward-based) infinite-horizon total regret of algorithm $\mathfrak{A}$ after $T$ steps in $M_{\infty}$, i.e., $\Delta_{\infty}(\mathfrak{A}, T, M_{\infty}) = T \rho^{\dagger} - \sum_{t=1}^{T} r_t$ where $\rho^{\dagger} := \max_{\pi} \rho_{\pi}(s)$ for all $s \in \mathcal{S}$. From Lem.\,\ref{lemma_equivalence_SSP_UCRL_constant_costs} we have $\rho^{\dagger} = \rho_{\infty}$. Moreover, since the rewards satisfy $r_{\infty} = \mathds{1}_{\overline{s}}$, we have $\sum_{t=1}^{T_K} r_t = K$. Putting everything together yields
\begin{align*}
\Delta(\mathfrak{A}, K) = T_K - (V^{\star}(s_0) + 1)K = (V^{\star}(s_0) + 1)( T_K \rho^{\dagger} - K) = (V^{\star}(s_0) + 1) \Delta_{\infty}(\mathfrak{A}, T_K, M_{\infty}).
\end{align*}

Note that $M_{\infty}$ is weakly-communicating, where its communicating set of states corresponds to all the states in $\mathcal{S}'$ that are accessible from $s_0$ with non-zero probability.
Although it is weakly-communicating, the specific reward structure, combined with the fact that rewards are necessarily known (since we consider the uniform-cost SSP setting and since the goal state $\overline{s}$ is assumed to be known), allows to run \UCRLtwo on this problem (see the Remark at the end of App.\,\ref{app:uniform_costs} for more detail).

Technically, \EVI~is guaranteed to converge since the associated extended MDP is weakly-communicating and by~\citet{puterman2014markov} it is sufficient for convergence of value iteration, see e.g.,~\citet[][Chap.\,9]{puterman2014markov} for finite action space or~\citet[][Thm. 1]{schweitzer1985undiscounted} for compact spaces.

From \citet[][Thm.\,2]{jaksch2010near} and using the anytime nature of \UCRLtwo, we have with probability at least $1-\delta$ for any $T > 1$,
\begin{align*}
    \Delta_{\infty}(\UCRLtwo, T, M_{\infty}) \leq 34 D_{\infty} S \sqrt{A T \log(\frac{T}{\delta})},
\end{align*}
where $D_{\infty} := \max_{s \neq s' \in \mathcal{S}'} \min_{\pi \in \Pi^{SD}(M_{\infty})} \mathbb{E}\left[ \tau_{\pi}(s \rightarrow s')\right]$ is the diameter of $M_{\infty}$. However, this bound may be vacuous since it depends on $D_{\infty}$ which may be equal to $+\infty$. By slightly changing the analysis of this result we can obtain an improved dependency on the SSP-diameter $D$. In particular it is sufficient to prove that for any \UCRLtwo episode $k$ and for any iteration $i$ of the optimal extended Bellman operator $L_{\mathcal{M}_k}$ (with $h_0 = 0$ and $h_i = (L_{\mathcal{M}_k})^i h_0$), we have that $\textrm{sp}(h_i) \leq D$ instead of the conventional upper bound $D_{\infty}$. The remainder of the proof shows this result. It is straightforward that $h_i(\overline{s}) \geq h_i(s)$ for any $s \in \mathcal{S}$ (this can be proved by recurrence on $i$ using the definition of $h_i = L_{\mathcal{M}_k} h_{i-1}$ and the fact that the reward in $\mathcal{M}_k$ is equal to $\mathds{1}_{\overline{s}}$). Introduce $\underline{s} \in \argmin_{s} h_i(s)$ and $\varphi_{\widetilde{M}}(\underline{s} \rightarrow \overline{s})$ the minimum expected shortest path from $\underline{s}$ to $\overline{s}$ in any MDP $\widetilde{M}$. Then from Lem.\,\ref{lemma_bound_span_bias} we have $\textrm{sp}(h_i) = h_i(\overline{s}) - h_i(\underline{s}) \leq \varphi_{\mathcal{M}_k}(\underline{s} \rightarrow \overline{s})$. Since the ``true'' MDP $M_{\infty} \in \mathcal{M}_k$, we have $\varphi_{\mathcal{M}_k}(\underline{s} \rightarrow \overline{s}) \leq \varphi_{M_{\infty}}(\underline{s} \rightarrow \overline{s})$. Furthermore, $\varphi_{M_{\infty}}(\underline{s} \rightarrow \overline{s}) = \varphi_{M}(\underline{s} \rightarrow \overline{s}) \leq D$. Putting everything together, we obtain that $\textrm{sp}(h_i) \leq D$. We thus have with probability at least $1-\delta$ for any $T > 1$,
\begin{align*}
    \Delta_{\infty}(\UCRLtwo, T, M_{\infty}) \leq 34 D S \sqrt{A T \log(\frac{T}{\delta})}.
\end{align*}
\end{proof}

While we would like to assess the dependency of the regret on the number of episodes $K$ (as in the finite-horizon case), the bound in Lem.\,\ref{lemma_regret_SSP_UCRL_constant_costs_communicating} contains the random total number of steps $T_K$ needed to reach $K$ episodes. In light of this, we derive in the following lemma an upper bound of $T_K$ that depends on the quantity of interest $K$. Plugging it in Lem.\,\ref{lemma_regret_SSP_UCRL_constant_costs_communicating} finally yields the result of Thm.\,\ref{theorem_regret_SSP_UCRL_constant_costs_communicating}.

\begin{lemma} Under the same event for which Lem.\,\ref{lemma_regret_SSP_UCRL_constant_costs_communicating} holds with probability at least $1-\delta$, we have
\begin{align*}
    T_K \leq 2 \left(V^{\star}(s_0)+1\right) K + \widetilde{O}\left( V^{\star}(s_0)^2 D^2 S^2 A \log\left(\frac{1}{\delta}\right) \right).
\end{align*}
\label{lemma_bound_Txm}
\end{lemma}

\begin{proof}
With probability at least $1-\delta$, we have from the proof of Lem.\,\ref{lemma_regret_SSP_UCRL_constant_costs_communicating} that
\begin{align*}
    T_K - \left(V^{\star}(s_0)+1\right) K \leq 34 \left(V^{\star}(s_0)+1\right) D S \sqrt{A T_K \log\left(\frac{T_K}{\delta}\right)}.
\end{align*}
This implies that
\begin{align*}
    T_K \leq 2\left(V^{\star}(s_0)+1\right) K \underbrace{- T_K + 68 \left(V^{\star}(s_0)+1\right) D S \sqrt{A T_K \log\left(\frac{T_K}{\delta}\right)}}_{:= (y)},
\end{align*}
where $(y)$ can be bounded using Lem.\,\ref{lemma_kazerouni_1} (with the constants $a_1 = 68 \left(V^{\star}(s_0)+1\right) D S \sqrt{A}$, $a_2 = \frac{1}{\delta}$ and $a_3=1$) as follows
\begin{align*}
    (y) \leq \frac{16}{9} \left(68 \left(V^{\star}(s_0)+1\right) D S \sqrt{A} \right)^2 \left[ \log\left(\frac{136 \left(V^{\star}(s_0)+1\right) D S \sqrt{A} e}{\sqrt{\delta}} \right)\right]^2.
\end{align*}

\end{proof}

\begin{lemma}
Consider an (extended) MDP $\widetilde{M}$ and define $L_{\widetilde{M}}$ as the associated optimal (extended) Bellman operator (of undiscounted value iteration). Given $h_0 = 0$ and $h_i = (L_{\widetilde{M}})^i h_0$ we have that
\begin{align*}
    \forall s_1, s_2 \in \mathcal{S}', h_i(s_2) - h_i(s_1) \leq r_{\max}\,\varphi_{\widetilde{M}}(s_1 \rightarrow s_2),
\end{align*}
where $\varphi_{\widetilde{M}}(s_1 \rightarrow s_2)$ is the minimum expected shortest path from $s_1$ to $s_2$ in $\widetilde{M}$ and $r_{\max}$ is the maximal state-action reward.
\label{lemma_bound_span_bias}
\end{lemma}
\begin{proof}
    The proof follows from the application of the argument of \citet[][Sect.\,4.3.1]{jaksch2010near}.
\end{proof}

\begin{lemma}[\citealp{kazerouni2017conservative}, Lem.\,8]
For any $x \geq 2$ and $a_1, a_2, a_3 > 0$, the following holds
\begin{align*}
    -a_3 x + a_1 \sqrt{x} \log(a_2 x) \leq \frac{16 a_1^2}{9 a_3} \left[ \log\left(\frac{2 a_1 \sqrt{a_2} e}{a_3} \right)\right]^2.
\end{align*}
\label{lemma_kazerouni_1}
\end{lemma}

\begin{figure}[t]
    \centering
    \begin{tikzpicture}
    \node[draw, circle] (s0) at (0,0) {$s_0$};
    \node[draw, circle] (s1) at (-2,0) {$s_1$};
    \node[draw, circle] (sb) at (2,0) {$\overline{s}$};
    \draw[->] (s1) -- (s0) node[midway, above] {$a_{10}$};
    \draw[->] (s0) -- (sb) node[midway, above] {$a_{01}$};
    \path[->] (s0) edge[loop above] (s0) node[above, yshift=22pt] {$a_{00}$};
    \path[->] (sb) edge[loop above] (sb) node[above, yshift=22pt] {$a_{\overline{s}0}$};
    \node[below of=s1, node distance=20pt] {$r(s_1)=0$};
    \node[below of=s0, node distance=20pt] {$r(s_0)=0$};
    \node[below of=sb, node distance=20pt] {$r(\overline{s})=1$};
    \end{tikzpicture}
    \caption{A toy example of SSP-communicating ($D=2$) reward-based MDP.}
    \label{fig:sspcom}
\end{figure}
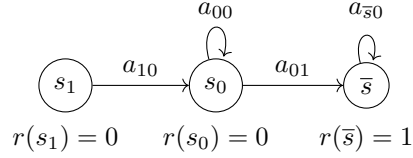
\begin{remark}
Consider the reward-based SSP $M$ in Fig.\,\ref{fig:sspcom}. $M$ is SSP-communicating while the associated MDP $M_\infty$ is weakly-communicating since $s_1$ is transient under every policy. There are just two possible deterministic policies: $\pi_0(s_0) = a_{00}$ and $\pi_1(s_0) = a_{01}$.
If rewards are unknown, \UCRLtwo will periodically alternate between policy $\pi_0$ and $\pi_1$ without converging to any of the two. This is due to the fact that, in the set of plausible MDPs $\mathcal{M}_k$ there will always be (i.e., $\forall k >0 $) an MDP with arbitrarily small but non-zero transition probability $\tilde{p}$ to state $s_1$, where, due to maximum uncertainty, there will be a self loop with probability $1$ and reward $r_{\max}$ (since $N_k(s_1,a_{10}) \in \{0,1\}$ depending on the initial state for any $k$). The probability $\tilde{p}$ will be sometimes higher for action $a_{00}$ and sometimes for $a_{01}$ depending on the counter $N_k$.
This is why \UCRLtwo will never converge.
However, if the rewards are known (which is always the case under Asm.\,\ref{assumption_uniform_costs} and as long as the goal state $\overline{s}$ is known), after a burn-in phase, it will be clear to \UCRLtwo that action $a_{00}$ is suboptimal. Even if there is probability $\tilde{p}>0$ to go to $s_1$, in $s_1$ the optimistic behaviour will be to go to $\overline{s}$ since it is the only one to provide reward. However, this imagined policy is suboptimal since it has an additional step and thus \UCRLtwo will select $\pi_1$. Note that while it is possible to make the MDP stochastic, this will lead to a longer burn-in phase but will not change the behaviour of \UCRLtwo in the long run.
\label{rmk_ucrl}
\end{remark}


\section{Proof of Lem.\,\ref{lemma_high_probability_intersection_bound_M_in_M_k}}

The proof is almost identical to the proof of \citet[][Thm.\,10]{improved_analysis_UCRL2B} and we report it below for completeness.

Recall that we define $\mathcal{M}_{k,j} := \{ \langle \mathcal{S}, \mathcal{A}, c, \widetilde{p} \rangle ~\vert ~ \widetilde{p} \in B_{k,j} \}$ to be the extended MDP defined by the confidence interval $B_{k,j} := \{ \widetilde{p} \in \mathcal{C} ~\vert ~ \widetilde{p}(\cdot \vert  \overline{s},a) = \mathds{1}_{\overline{s}} ~\textrm{and}~ \forall (s,a) \in \mathcal{S} \times \mathcal{A}, \norm{\widetilde{p}(\cdot \vert  s,a) - \widehat{p}_{k,j}(\cdot \vert  s,a)}_1 \leq \beta_{k,j}(s,a)\}$, with $\mathcal{C}$ the $S'$-dimensional simplex and
\begin{align*}
\beta_{k,j}(s,a) &:= \sqrt{\frac{8S\log\left(\frac{2 AN_{k,j}^+(s,a)}{\delta}\right)}{N_{k,j}^+(s,a)}}.
\end{align*}
Furthermore we introduce $B_{k,j}(s,a) := \{\widetilde{p} \in \mathcal{C} : \|\widetilde{p}(\cdot \vert s,a) - \widehat{p}_{k,j}(\cdot \vert s,a)\|_1 \leq \beta_{k,j}(s,a)\}$ (and similarly for $B_{k,j}(s,a,s')$).
We want to bound the probability of event $\mathcal{E}^{\mathcal{C}} := \bigcup_{k=1}^{+\infty} \bigcup_{j=1}^{J_k} \left\{M \not\in {\mathcal{M}}_{k,j} \right\}$.
As explained by \citet[][Chap.\,5]{bandittorcsaba}, when $(s,a)$ is visited for the $n$-th times, the next state that we observe is the $n$-th element of an infinite sequence of i.i.d.\ r.v. lying in $\mathcal{S'}$ with probability density function $p(\cdot \vert s,a)$. In \UCRLtwo \citep{jaksch2010near}, the sample means $\widehat{p}_{k,j}$ and the confidence intervals $B_{k,j}$ are defined as depending on $(k,j)$. Actually, these quantities depend only on the first $N_{k,j}(s,a)$ elements of the infinite i.i.d.\ sequences that we just mentioned. For the rest of the proof, we will therefore slightly change our notations and denote by $\widehat{p}_n(s'\vert s,a)$ and $B_n(s'\vert s,a)$ the sample means and confidence intervals after the first $n$ visits in $(s,a)$. Thus, the r.v.\ that we denoted by $\widehat{p}_{k,j}$ actually corresponds to $\widehat{p}_{N_{k,j}(s,a)}$ with our new notation (and similarly for $B_{k,j}$). This change of notation will make the proof easier.

If $M \not\in \mathcal{M}_{k,j}$, then there exists a $k\geq 1$ and $j \geq 0$ s.t.\ $p(\cdot \vert s,a) \not\in B_{N_{k,j}(s,a)}(s,a)$ for at least one $(s,a,s') \in \mathcal{S} \times \mathcal{A} \times \mathcal{S'}$. This means that there exists at least one value $n\geq 0$ s.t.\ $p(s'\vert s,a) \not\in B_n(s,a,s')$. Consequently we have the following inclusion
\begin{align*}
 \mathcal{E}^{\mathcal{C}} \subseteq \bigcup_{s,a} \bigcup_{n=0}^{+\infty} \left\{ p(\cdot \vert s,a) \not\in B_{n}(s,a) \right\}.
\end{align*}
Using Boole's inequality we have
\begin{align*}
        \mathbb{P}(\mathcal{E}^{\mathcal{C}}) \leq \sum_{s,a} \sum_{n=0}^{+\infty} \mathbb{P}(p(\cdot \vert s,a) \not\in B_n(s,a) ).
\end{align*}
Let us fix a tuple $(s,a) \in \mathcal{S} \times \mathcal{A}$ and define for all $n\geq 0$
 \begin{align*}
         \epsilon_{n}(s,a) := \sqrt{\frac{2 \log\left((2^{S'} - 2) 5 SA (n^+)^2/\delta\right) }{n^+}},
 \end{align*}
where $n^+ := \max\{n,1\}$.
Since $S' = S+1 \leq 2 S$, it is immediate to verify that almost surely, $\epsilon_{n}(s,a) \leq \beta_{n}(s,a)$.
Using Weissman's inequality~\citep{weissman2003inequalities, jaksch2010near} we have that for all $n \geq 1$
 \begin{align*}
         \mathbb{P}( \| p(\cdot \vert s,a) - \widehat{p}_n(\cdot \vert s,a) \|_1 \geq \beta_{n}(s,a) ) &\leq \mathbb{P}(\| p(\cdot \vert s,a) - \widehat{p}_n(\cdot \vert s,a) \|_1 \geq \epsilon_{n}(s,a) ) \leq \frac{\delta}{5 n^2SA}.
 \end{align*}
Note that when $n = 0$ (i.e., when there has not been any observation of $(s,a)$), $\epsilon_{0}(s,a) \geq 2$ so $\mathbb{P} (\| p(\cdot \vert s,a) - \widehat{p}_0(\cdot \vert s,a) \|_1 \geq \epsilon_{0}(s,a)) = 0$ by definition.
As a result, we have that for all $n \geq 1$
\begin{align*}
 \mathbb{P}( p(\cdot \vert s,a) \notin B_n(s,a) ) \leq \frac{\delta}{5 n^2SA},
\end{align*}
and this probability is equal to $0$ if $n=0$. Finally we obtain
\begin{align*}
 \mathbb{P} \left( \exists k \geq 1, \exists j \in [0, J_k], \text{ s.t.\ } M \not\in {\mathcal{M}}_{k,j} \right) \leq \sum_{s,a} \left(0 + \sum_{n=1}^{+\infty} \frac{\delta}{5 n^2 SA} \right) = \frac{\pi^2 \delta}{30} \leq \frac{\delta}{3},
\end{align*}
which concludes the proof.




\section{Proof of Lem.\,\ref{regret_key_lemma}}
\label{app:proof_regret_key_lemma}

For notational ease, in Sect.\,\ref{app:proof_regret_key_lemma} we adopt the notation $H_{k} :=H_{k,0}$, $\wt{\pi}_{k} := \wt{\pi}_{k,0}$, $\varepsilon_{k} := \varepsilon_{k,0}$ (i.e., we remove the subscript $0$).

Furthermore, for any $k \in [K]$ and $h \in [H_k]$, we denote by $s_{k,h}$ the state visited in the $h$-th step of episode $k$.

Assume from now on that the event $\mathcal{E}$ holds. From Lem.\,\ref{lemma_optimism} we have
\begin{align*}
    \mathcal{W}_K = \sum_{k=1}^K \left[ \left( \sum_{h=1}^{H_k} c(s_{k,h}, \wt{\pi}_k(s_{k,h})) \right) - V^{\star}(s_0) \right] \leq \sum_{k=1}^K \left[ \left( \sum_{h=1}^{H_k} c(s_{k,h}, \wt{\pi}_k(s_{k,h})) \right) - \widetilde{v}_k(s_{0}) \right] = \sum_{k=1}^K  \Theta_{k,1}(s_{k,1}),
\end{align*}
where $s_{k,1} := s_0$, and for any $k \in [K]$ and $h \in [H_k]$, we introduce
\begin{align*}
    \Theta_{k,h}(s_{k,h}) := \sum_{t=h}^{H_k} c(s_{k,t}, \wt{\pi}_k(s_{k,t})) - \widetilde{v}_{k}(s_{k,h}).
\end{align*}

For any $h \in [H_k - 1]$, we introduce
\begin{align*}
    \Phi_{k,h} := \widetilde{v}_k(s_{k,h+1}) - \sum_{y \in \mathcal{S}} p(y\, \vert \,  s_{k,h},\wt{\pi}_k(s_{k,h})) \widetilde{v}_k(y).
\end{align*}
We then have
\begin{align}
\nonumber
    \Theta_{k,h}(s_{k,h}) &= \sum_{t=h}^{H_k} c(s_{k,t}, \wt{\pi}_k(s_{k,t})) - \widetilde{v}_{k}(s_{k,h}) \\
\nonumber
    &\leq \sum_{t=h}^{H_k} c(s_{k,t}, \wt{\pi}_k(s_{k,t})) - \widetilde{\mathcal{L}}_{k} \widetilde{v}_{k}(s_{k,h}) + \varepsilon_k\\
\nonumber
    &\myeqa \sum_{t=h}^{H_k} c(s_{k,t}, \wt{\pi}_k(s_{k,t})) - c(s_{k,h},\wt{\pi}_k(s_{k,h})) - \sum_{y \in \mathcal{S}} \widetilde{p}_k(y \, \vert \,  s_{k,h},\wt{\pi}_k(s_{k,h})) \widetilde{v}_k(y) + \varepsilon_k\\
\nonumber
    &= \sum_{t=h+1}^{H_k} c(s_{k,t}, \wt{\pi}_k(s_{k,t})) - \sum_{y \in \mathcal{S}} [\widetilde{p}_k(y \, \vert \,  s_{k,h},\wt{\pi}_k(s_{k,h})) - p(y\, \vert \,  s_{k,h},\wt{\pi}_k(s_{k,h})) + p(y\, \vert \,  s_{k,h},\wt{\pi}_k(s_{k,h}))] \widetilde{v}_k(y) + \varepsilon_k\\
\nonumber
    &\myineeqb \sum_{t=h+1}^{H_k} c(s_{k,t}, \wt{\pi}_k(s_{k,t})) - \sum_{y \in \mathcal{S}} p(y\, \vert \,  s_{k,h},\wt{\pi}_k(s_{k,h})) \widetilde{v}_k(y) \\
\nonumber
    &\quad + \norm{p(\cdot \, \vert  s_{k,h},\wt{\pi}_k(s_{k,h})) - \widetilde{p}_k(\cdot \, \vert  s_{k,h}, \wt{\pi}_k(s_{k,h}))}_{1} \norm{\widetilde{v}_k}_{\infty} + \varepsilon_k\\
\nonumber
    &\myineeqc \sum_{t=h+1}^{H_k} c(s_{k,t}, \wt{\pi}_k(s_{k,t})) - \sum_{y \in \mathcal{S}} p(y\, \vert \,  s_{k,h},\wt{\pi}_k(s_{k,h})) \widetilde{v}_k(y) + 2 \beta_k(s_{k,h}, \wt{\pi}_k(s_{k,h})) c_{\max} D + \varepsilon_k\\
\nonumber
    &= \Theta_{k,h+1}(s_{k,h+1}) + \widetilde{v}_k(s_{k,h+1}) - \sum_{y \in \mathcal{S}} p(y\, \vert \,  s_{k,h},\wt{\pi}_k(s_{k,h})) \widetilde{v}_k(y) + 2 \beta_k(s_{k,h}, \wt{\pi}_k(s_{k,h})) c_{\max} D + \varepsilon_k\\
    &= \Theta_{k,h+1}(s_{k,h+1}) + \Phi_{k,h} + 2 \beta_k(s_{k,h}, \wt{\pi}_k(s_{k,h})) c_{\max} D + \varepsilon_k,
    \label{eq:Theta.expansion1}
\end{align}
where (a) stems from the fact that $\wt{\pi}_k$ is the greedy policy with respect to $(\widetilde{v}_{k}, \varepsilon_k)$, (b) leverages that $\wt{v}_k \geq 0$ component-wise and (c) combines Lem.\,\ref{lemma_optimism} and \ref{lemma_1}. Furthermore, whatever the value of $s_{k,H_k}$ we have
\begin{align*}
    \Theta_{k,H_k}(s_{k,H_k}) &= c(s_{k,H_k}, \wt{\pi}_k(s_{k,H_k})) - \widetilde{v}_k(s_{k, H_k}) \\
    &\leq c(s_{k,H_k}, \wt{\pi}_k(s_{k,H_k})) - \widetilde{\mathcal{L}}_{k} \widetilde{v}_k(s_{k,H_k}) + \varepsilon_k\\
    &= c(s_{k,H_k}, \wt{\pi}_k(s_{k,H_k})) - c(s_{k,H_k}, \wt{\pi}_k(s_{k,H_k})) - \sum_{y \in \mathcal{S}} \widetilde{p}_k(y \, \vert \,  s_{k,H_k},\wt{\pi}_k(s_{k,H_k})) \underbrace{\widetilde{v}_k(y)}_{\geq 0} + \varepsilon_k \\
    &\leq \varepsilon_k.
\end{align*}

By telescopic sum we get (using Eq.\,\ref{eq:Theta.expansion1})
\begin{align*}
    \Theta_{k,1}(s_{k,1}) &= \sum_{h=1}^{H_k-1} (\Theta_{k,h}(s_{k,h}) - \Theta_{k,h+1}(s_{k,h+1})) + \Theta_{k,H_k}(s_{k,H_k}) \\
    &\leq \sum_{h=1}^{H_k-1} \Phi_{k,h} + 2 c_{\max} D \sum_{h=1}^{H_k-1} \beta_k(s_{k,h},\wt{\pi}_k(s_{k,h})) + (H_k-1) \varepsilon_k + \Theta_{k,H_k}(s_{k,H_k}) \\
    &\leq \sum_{h=1}^{H_k-1} \Phi_{k,h} + 2 c_{\max} D \sum_{h=1}^{H_k-1} \beta_k(s_{k,h},\wt{\pi}_k(s_{k,h})) + H_k \varepsilon_k.
\end{align*}

Summing over the episode index $k$ yields
\begin{align*}
    \sum_{k=1}^K  \Theta_{k,1}(s_{k,1}) \leq \underbrace{\sum_{k=1}^K \sum_{h=1}^{H_k-1} \Phi_{k,h}}_{:= X_K} + 2 c_{\max} D \underbrace{\sum_{k=1}^K \sum_{h=1}^{H_k-1} \beta_k(s_{k,h},\wt{\pi}_k(s_{k,h}))}_{:= Y_K} + \underbrace{\sum_{k=1}^K H_k \varepsilon_k}_{:= Z_K}.
\end{align*}

In order to bound $X_K$, we can write
\begin{align*}
    &\mathbb{P}\left( \sum_{k=1}^K \sum_{h=1}^{H_k - 1} \Phi_{k,h} \geq 2c_{\max}D \sqrt{2 \left(\sum_{k=1}^K H_k \right) \log\left(\frac{2 \left(\sum_{k=1}^K H_k \right)^2}{\delta}\right)} \right) \\
    &\leq \sum_{n=1}^{+\infty} \mathbb{P}\left( \sum_{k=1}^K \sum_{h=1}^{H_k} \Phi_{k,h} \geq 2c_{\max}D \sqrt{2 n \log\left(\frac{2 n^2}{\delta}\right)} ~ \bigcap ~ \sum_{k=1}^K H_k = n\right) \\
    &\leq \sum_{n=1}^{+\infty} \mathbb{P}\left( \sum_{t=1}^n \widetilde{\Phi}_t \geq 2c_{\max}D \sqrt{2 n \log\left(\frac{2 n^2}{\delta}\right)} \right),
\end{align*}
where we introduce for any $t > 0$, $$\widetilde{\Phi}_t = \left\{
    \begin{array}{cl}
        \Phi_{\widetilde{k}_t,t - Z_t} & \mbox{if~} t > Z_t, \\
        \Phi_{\widetilde{k}_t+1,1} & \mbox{otherwise},
    \end{array}
\right.$$
where $\widetilde{k}_t = \max{ \{ k ~\vert~ \sum_{k'=1}^{k} H_{k'} \leq t \} }$ and $Z_t = \sum_{k'=1}^{\widetilde{k}_t-1} H_{k'} + 1$, i.e., we map a value $t$ to the double index $(k,h)$. Denote by $\mathcal{G}_q$ the history of all random events up to (and including) step $h$ of episode $k$ (i.e.,~$q = \sum_{k'=1}^{k-1} H_k + h$). We have $\mathbb{E}\left[\Phi_{k,h} \vert \mathcal{G}_q \right] = 0$ (since $\widetilde{v}_k(\overline{s})=0$), and furthermore the stopping time $H_k$ is selected at the beginning of episode $k$ so it is adapted w.r.t.\,$\mathcal{G}_q$. Hence, $(\widetilde{\Phi}_t)$ is a martingale difference sequence, such that $\abs{\widetilde{\Phi}_t} \leq 2 c_{\max} D$. For any fixed $n > 0$, we thus have from Azuma-Hoeffding's inequality that
\begin{align*}
    \mathbb{P}\left( \sum_{t=1}^n \widetilde{\Phi}_t \geq 2c_{\max}D \sqrt{2 n \log\left(\frac{2 n^2}{\delta}\right)} \right) \leq \frac{\delta}{2 n^2}.
\end{align*}
As a result, from a union bound over all possible values of $n > 0$, we have with probability at least $1-\frac{2\delta}{3}$,
\begin{align}
    \sum_{k=1}^K \sum_{h=1}^{H_k - 1} \Phi_{k,h} \leq 2c_{\max}D \sqrt{2 \left(\sum_{k=1}^K H_k \right) \log\left(\frac{3 \left(\sum_{k=1}^K H_k \right)^2}{\delta}\right)}.
\label{azuma_martingale}
\end{align}

We now proceed in bounding $Y_K$ using a pigeonhole principle. Denoting by $N^{(1)}$ the counter of samples \textit{only} collected during attempts in phase \ding{172}, we get
\begin{align*}
    \sum_{k=1}^K \sum_{h=1}^{H_k-1} \sqrt{\frac{1}{N_k^{(1)}(s_{k,h},\wt{\pi}_k(s_{k,h}))}} \leq \sum_{s,a} \sum_{n=1}^{N_K^{(1)}(s,a)} \sqrt{\frac{1}{n}} \leq \sum_{s,a} 2 \sqrt{ N_K^{(1)}(s,a) } \leq 2 \sqrt{SA} \sqrt{\sum_{s,a} N_K^{(1)}(s,a)} \leq 2 \sqrt{S A T_{K,1}}.
\end{align*}

We have $N_k^+(s,a) \geq N_k^{(1)+}(s,a)$ so by applying the technical Lem.\,\ref{technical_lemma_function} (and considering that $A \geq 2$ since if $A=1$ there is no learning problem), we get
\begin{align*}
    \beta_k(s,a) = \sqrt{\frac{8S\log\left(\frac{2AN_k^+(s,a)}{\delta}\right)}{N_k^+(s,a)}} \leq \sqrt{\frac{8S\log\left(\frac{2A N_k^{(1)+}(s,a)}{\delta}\right)}{N_k^{(1)+}(s,a)}}.
\end{align*}
Therefore we obtain
\begin{align}
    \sum_{k=1}^K \sum_{h=1}^{H_k-1} \beta_k(s_{k,h},\wt{\pi}_k(s_{k,h})) \leq 2 S \sqrt{8 A T_{K,1} \log\left(\frac{2 A T_{K,1}}{\delta}\right)}.
\label{pigeonhole_principle}
\end{align}

We finally bound $Z_K$. We have for any $k \in [K]$, $H_k \leq \Omega_K$ and we select $\varepsilon_k = \frac{c_{min}}{2 t_{k,0}}$, hence we have $T_{K,1} \leq \Omega_K K$ and
\begin{align*}
    \sum_{k=1}^K H_k \varepsilon_k \leq \frac{c_{\min}}{2} \sum_{t=1}^{T_{K,1}} \frac{\Omega_K}{t} \leq \frac{c_{\min}}{2} \Omega_K \left(1 + \log(\Omega_K K)\right).
\end{align*}

Putting everything together, a union bound and Lem.\,\ref{lemma_high_probability_intersection_bound_M_in_M_k} yields with probability at least $1-\delta$,
\begin{align*}
     \sum_{k=1}^K \left[ \left( \sum_{h=1}^{H_k} c(s_{k,h}, \wt{\pi}_k(s_{k,h})) \right) - \widetilde{v}_k(s_{0}) \right] &\leq 4c_{\max} D S \sqrt{8 A T_{K,1} \log\left(\frac{2 A T_{K,1}}{\delta}\right)} \\ & \quad + 2c_{\max}D \sqrt{2 T_{K,1} \log\left(\frac{3 T_{K,1}^2}{\delta}\right)} + \frac{c_{\min}}{2} \Omega_K \left(1 + \log(\Omega_K K)\right).
\end{align*}

\begin{lemma}
    For any constant $c \geq 4$, the function $f(x) := \sqrt{ \frac{\log(cx)}{x}}$ is a non-increasing function for $x \geq 1$.
\label{technical_lemma_function}
\end{lemma}

\begin{proof}
    Introduce the function $g(x) := f(x)^2$. We have $g'(x) = \frac{1 - \log(cx)}{x^2} \leq 0$ since $x \geq 1 \geq \frac{e}{c}$. So $g$ is non-increasing, hence by composition of functions, $f = \sqrt{g}$ is also non-increasing.
\end{proof}

Interestingly, the bound of Lem.\,\ref{regret_key_lemma} resembles a combination of finite- and infinite-horizon guarantees. On the one hand, we have the standard dependency of finite-horizon problems on the horizon $H$ and number of episodes $K$. On the other hand, $H$ is no longer bounding the range of the value functions, which is replaced by $c_{\max} D$ as in infinite-horizon problems.


\section{Proof of Lem.\,\ref{lemma_bound_H_k_analysis_1}}
\label{app:bound.maximal.episode.length}

We start the proof of Lem.\,\ref{lemma_bound_H_k_analysis_1} by deriving a general result --- which may be of independent interest --- that \textit{upper bounds the moments of any discrete PH distribution}.\footnote{Note that while there actually exists a \textit{closed-form} expression of the moments of a continuous PH distribution (see e.g.,~\citealp[][Eq.\,2.13]{latouche1999introduction}), it does not extend to the discrete case.}

\begin{lemma} Consider an absorbing Markov Chain with state space $\mathcal{Y} \cup \{ \overline{y} \}$, a single absorbing state $\overline{y}$ and $\abs{\mathcal{Y}}$ transient states. Denote by $Q \in \mathbb{R}^{Y \times Y}$ the transition matrix within the states in $\mathcal{Y}$ and by $\tau (y) := \tau (y\rightarrow \overline{y})$ the first hitting time of state $\overline{y}$ starting from state $y$. Suppose that there exists a constant $\lambda \geq 2$ such that for any state $y \in \mathcal{Y}$, we have $\mathbb{E}\left[\tau(y \rightarrow \overline{y}) \right] \leq \lambda$.
Then for any $r \geq 1$ and any state $y \in \mathcal{Y}$, we have
\begin{align*}
    \mathbb{E}\left[\tau(y)^r\right] \leq 2 \left( r \lambda \right)^{r}.
\end{align*}
\label{lem:raw_moments_discrete_ph_distribution}
\end{lemma}

\begin{proof}
We first leverage a closed-form expression of the \textit{factorial moments} of discrete PH distributions. For any $r\geq 1$, denoting by $(\tau)_r$ the $r$-th factorial moment of $\tau$, i.e.,~$(\tau)_r := \tau (\tau - 1)...(\tau - r + 1)$, we have (see e.g.,~\citealp[][Eq.\,2.15]{latouche1999introduction}) that for any starting state $y \in \mathcal{Y}$,
\begin{align*}
    \mathbb{E}\left[ (\tau)_r(y) \right] = r! \mathds{1}_{y}^\top (I - Q)^{-r} Q^{r-1} \mathds{1}.
\end{align*}
Recalling that the $\norm{\cdot}_{\infty}$ (resp.\,$\norm{\cdot}_{1}$) norm of a matrix is equal to its maximum absolute row (resp.\,column) sum, we have by H{\"o}lder's inequality, for any $j \in [r]$,
\begin{align}
     \mathbb{E}\left[ (\tau)_j(y) \right] &= j! \left\langle (\mathds{1}_{y}^\top (I - Q)^{-j})^\top, Q^{j-1} \mathds{1} \right\rangle \nonumber \\
    &\leq j! \norm{(\mathds{1}_y^\top (I - Q)^{-j})^\top}_{1} \norm{Q^{j-1} \mathds{1}}_{\infty} \nonumber \\
    &= j! \norm{((I - Q)^{-j})^\top \mathds{1}_y}_{1} \norm{Q^{j-1} \mathds{1}}_{\infty} \nonumber \\
    &\leq j! \norm{((I-Q)^{-j})^\top}_{1} \norm{\mathds{1}_y}_{1} \norm{Q^{j-1}}_{\infty} \norm{\mathds{1}}_{\infty} \nonumber \\
    &\leq j! \norm{(I-Q)^{-j}}_{\infty} \norm{Q^{j-1}}_{\infty} \nonumber \\
    &\leq j! \norm{(I-Q)^{-1}}_{\infty}^{j},
\label{ineq_moments_1}
\end{align}
where the last inequality uses the fact that $\norm{Q^{j-1}}_{\infty} \leq 1$ since the matrix $Q^{j-1}$ is substochastic. There remains to upper bound the quantity $\norm{(I - Q) ^{-1}}_{\infty}$. Consider a state $$z \in \argmax_{y \in \mathcal{Y}} \sum_{y' \in \mathcal{Y}} (I - Q)^{-1}_{y y'}.$$ By choice of $z$ and non-negativity of the matrix $(I-Q)^{-1}$, we have
\begin{align*}
\norm{(I - Q) ^{-1}}_{\infty} &= \sum_{y' \in \mathcal{Y}} \abs{(I - Q)^{-1}_{z y'}} \\ &= \sum_{y' \in \mathcal{Y}} (I - Q)^{-1}_{z y'} \\ &=\mathds{1}_z^\top (I - Q)^{-1} \mathds{1} \\
&= \sum_{n=0}^{\infty} \mathds{1}_z^\top Q^n \mathds{1}.
\end{align*}
Since $\tau(z)$ follows a discrete PH distribution, we have from Lem.\,\ref{equation_discrete_phase_type_distribution} that
\begin{align*}
    \mathds{1}_z^\top Q^n \mathds{1} = \mathbb{P}(\tau(z) > n).
\end{align*}
Consequently,
\begin{align}
\norm{(I - Q) ^{-1}}_{\infty} = \sum_{n=0}^{\infty} \mathbb{P}(\tau(z) > n) = \mathbb{E}[\tau(z)] \leq \lambda.
\label{ineq_bound_infinite_norm}
\end{align}
Plugging Eq.\,\eqref{ineq_bound_infinite_norm} into Eq.\,\eqref{ineq_moments_1} thus yields for any $y \in \mathcal{Y}$,
\begin{align*}
     \mathbb{E}\left[ (\tau)_j(y) \right] \leq j! \lambda^j.
\end{align*}
Furthermore, the (raw) moment of a random variable can be expressed in terms of its factorial moments by the following formula (see e.g.,~\citealp[][Eq.\,3.1]{mahmoodinductive})
\begin{align*}
    \mathbb{E}\left[ \tau(y)^r \right] = \sum_{j=1}^r \left\{ {r \atop j} \right\} \mathbb{E}\left[ (\tau)_j(y) \right],
\end{align*}
where the curly braces denote Stirling numbers of the second kind, i.e.,
\begin{align*}
\left\{ {r \atop j} \right\} := \frac{1}{j!}\sum_{i=0}^{j}(-1)^{j-i}{j \choose i} i^r.
\end{align*}
Using the upper bound (see e.g.,~\citealp[][Eq.\,9]{canfield2002problem})
\begin{align*}
    \left\{ {r \atop j} \right\} \leq \frac{j^r}{j!},
\end{align*}
we obtain
\begin{align*}
    \mathbb{E}\left[ \tau(y)^r \right] &\leq \sum_{j=1}^r j^r \lambda^j.
\end{align*}
We conclude the proof of Lem.\,\ref{lem:raw_moments_discrete_ph_distribution} with the fact that
\begin{align*}
\sum_{j=1}^r j^r \lambda^j \leq r^r \sum_{j=1}^r \lambda^j \leq r^r \lambda \frac{\lambda^r - 1}{\lambda - 1} \leq r^r 2 \lambda^{r},
\end{align*}
where the last inequality holds since $\lambda \geq 2$.
\end{proof}

We are now ready to prove Lem.\,\ref{lemma_bound_H_k_analysis_1}.

For notational ease, in Sect.\,\ref{app:bound.maximal.episode.length} we adopt the notation $H_{k} :=H_{k,0}$, $\wt{\pi}_{k} := \wt{\pi}_{k,0}$, $\varepsilon_{k} := \varepsilon_{k,0}$ (i.e., we remove the subscript $0$).

Denote by $\mathcal{G}_{k-1}$ the history of all random events up to (and including) episode $k-1$. In this section as well as in Sect.\,\ref{app:bound.failure.events}, we will write $\mathbb{E}\left[ \mathds{1}_{\{ \tau_{\pi}^{p}(s) > H_k-1 \}} \, \vert \,  \mathcal{G}_{k-1} \right] = \mathbb{P}(\tau_{\pi}^{p}(s) > H_k-1)$, i.e.\,the probability $\mathbb{P}$ is only over the randomization of the sequence of states generated by the policy $\pi$ in the model $p$ starting from state $s$ (i.e., it is conditioned on $\mathcal{G}_{k-1}$, the policy $\pi$, the model $p$ and the starting state $s$).

Suppose that the event $\mathcal{E}$ holds and fix an episode $k \in [K]$. Denote by $\wt{Q} := Q_{\widetilde{\pi}_k}^{\widetilde{p}_k}$ the optimistic transition matrix within $\mathcal{S}$ of policy $\widetilde{\pi}_k$ in the transition model $\widetilde{p}_k$. Also, for any state $s \in \mathcal{S}$, denote by $\widetilde{\tau}(s) := \tau_{\widetilde{\pi}_k}^{\widetilde{p}_k}(s)$ the hitting time of $\overline{s}$ starting from $s$ following policy $\widetilde{\pi}_k$ in the transition model $\widetilde{p}_k$.

We introduce the Bellman operator $\mathcal{T}_{\varepsilon_k}^{\wt{\pi}_k}$ for policy $\wt{\pi}_k$, that verifies for any vector $v \in \mathbb{R}^S$ and state $s \in \mathcal{S}$,
\begin{align*}
\mathcal{T}_{\varepsilon_k}^{\wt{\pi}_k} v(s) := c(s, \wt{\pi}_k(s)) - \varepsilon_k + \sum_{y \in \mathcal{S}} \wt{p}_k(y  \, \vert \,  s,\wt{\pi}_k(s)) v(y),
\end{align*}
i.e., it corresponds to the operator $\widetilde{\mathcal{L}}^{\wt{\pi}_{k}}_k$ with $\varepsilon_k$ subtracted to all the costs. Note that its costs are all positive by choice of $\varepsilon_k = \frac{c_{\min}}{2 t_k}$. Combining Lem.\,\ref{lemma_optimism} and the fact that $\wt{\pi}_k$ is the greedy policy w.r.t.\,$\wt{v}_k$ yields that $\widetilde{\mathcal{L}}^{\wt{\pi}_{k}}_k \widetilde{v}_{k} = \widetilde{\mathcal{L}}_{k} \widetilde{v}_{k} \leq \widetilde{v}_{k} + \varepsilon_k$. Consequently, we have the following component-wise inequality
\begin{align*}
    \mathcal{T}_{\varepsilon_k}^{\wt{\pi}_k} \wt{v}_k \leq \wt{v}_k.
\end{align*}
By monotonicity of the operator $\mathcal{T}_{\varepsilon_k}^{\wt{\pi}_k}$~\citep{puterman2014markov,bertsekas1995dynamic}, we have for all $m > 0$,
\begin{align*}
    (\mathcal{T}_{\varepsilon_k}^{\wt{\pi}_k})^m \wt{v}_k \leq \wt{v}_k,
\end{align*}
and hence taking the limit $m \rightarrow + \infty$ yields $\wt{U}_{\wt{\pi}_k, \varepsilon_k} \leq \wt{v}_k$, where $\wt{U}_{\wt{\pi}_k, \varepsilon_k}$ is defined as the value function of policy $\wt{\pi}_k$ in the model $\wt{p}_k$ with $\varepsilon_k$ subtracted to all the costs, i.e.,
\begin{align*}
    \wt{U}_{\wt{\pi}_k, \varepsilon_k}(s) &:= \mathbb{E}_{\wt{p}_k}\left[ \sum_{t=1}^{\wt{\tau}(s)} \left( c(s_t, \wt{\pi}_k(s_t)) - \varepsilon_k \right) ~\mid~ s_1 = s \right] = \wt{V}_{\wt{\pi}_k}(s) - \varepsilon_k \mathbb{E}\left[\wt{\tau}(s) \right],
\end{align*}
where $\wt{V}_{\wt{\pi}_k}(s) := \mathbb{E}_{\wt{p}_k}\left[ \sum_{t=1}^{\wt{\tau}(s)} c(s_t, \wt{\pi}_k(s_t)) ~\mid~ s_1 = s \right]$ is the value function of policy $\wt{\pi}_k$ in the model $\wt{p}_k$. Since $\wt{V}_{\wt{\pi}_k}(s) \geq c_{\min} \mathbb{E}[\widetilde{\tau}(s)]$, we have
\begin{align*}
    (c_{\min} - \varepsilon_k) \mathbb{E}[\widetilde{\tau}(s)] \leq \wt{V}_{\wt{\pi}_k}(s) - \varepsilon_k \mathbb{E}\left[\wt{\tau}(s) \right] = \wt{U}_{\wt{\pi}_k, \varepsilon_k}(s) \leq \wt{v}_k(s).
\end{align*}

Using successively 
the above inequality, the fact that $\varepsilon_k \leq \frac{c_{\min}}{2}$, Lem.\,\ref{lemma_optimism} and \ref{lemma_1}, we obtain for any $s \in \mathcal{S}$,
\begin{align}
    \mathbb{E}[\widetilde{\tau}(s)] \leq \frac{\wt{v}_k(s)}{c_{\min} - \varepsilon_k} \leq \frac{2V^{\star}(s)}{c_{\min}} \leq \frac{2c_{\max} D}{c_{\min}}.
\label{cost_weighted_optimism_expected_hitting_time}
\end{align}

Fix any $r \geq 1$ and $s \in \mathcal{S}$. According to a corollary of Markov's inequality (since $x \mapsto x^r $ is a monotonically increasing non-negative function for the non-negative reals), we have
\begin{align*}
    \mathbb{P}(\widetilde{\tau}(s) \geq H_k - 1) \leq \frac{\mathbb{E}\left[\widetilde{\tau}(s)^r\right]}{(H_k - 1)^r}.
\end{align*}

We can apply Lem.\,\ref{lem:raw_moments_discrete_ph_distribution} to the discrete PH distribution $\wt{\tau}$ with the choice of $\lambda := \frac{2c_{\max}D}{c_{\min}}$ guaranteed by Eq.\,\eqref{cost_weighted_optimism_expected_hitting_time}. This yields
\begin{align*}
    \mathbb{E}\left[\widetilde{\tau}(s)^r\right] \leq 2 \left( r \frac{2c_{\max}D}{c_{\min}} \right)^{r}.
\end{align*}
Hence, we have
\begin{align}
    \mathbb{P}(\widetilde{\tau}(s) \geq H_k - 1) \leq \frac{ 2 \left( r \frac{2c_{\max}D}{c_{\min}} \right)^{r} }{(H_k - 1)^r}.
\label{ineq_101}
\end{align}

There exists $y \in \mathcal{S}$ such that
\begin{align}
\norm{\wt{Q}^{H_k-2}}_{\infty} = \mathds{1}_y^\top \wt{Q}^{H_k-2} \mathds{1} = \mathbb{P}(\widetilde{\tau}(y) > H_k - 2) = \mathbb{P}(\widetilde{\tau}(y) \geq H_k - 1),
\label{norm_satisfied_one_state_y}
\end{align}
where the before-last equality uses Lem.\,\ref{equation_discrete_phase_type_distribution} applied to $\wt{\pi}_k \in \Pi^{PSD}(\langle\mathcal{S}', \mathcal{A}, c, \wt{p}_k, y \rangle)$ (the fact that $\wt{\pi}_k$ is proper in $\wt{p}_k$ stems from Eq.\,\ref{cost_weighted_optimism_expected_hitting_time}), while the last equality uses that the hitting time $\wt{\tau}(y)$ is an integer. 
By definition of $H_k := \min \left\{ n > 1 : \norm{\wt{Q}^{n-1}}_{\infty} \leq \frac{1}{\sqrt{k}} \right\}$, we have $\norm{\wt{Q}^{H_k-2}}_{\infty} > \frac{1}{\sqrt{k}}$. Combining this with Eq.\,\eqref{ineq_101} and \eqref{norm_satisfied_one_state_y} yields
\begin{align*}
    \frac{ 2 \left( r \frac{2c_{\max}D}{c_{\min}} \right)^{r}}{(H_k-1)^r} > \frac{1}{\sqrt{k}},
\end{align*}
which implies that
\begin{align*}
 H_k - 1 < r \frac{2c_{\max}D}{c_{\min}} \left( 2 \sqrt{k} \right)^{\frac{1}{r}}.
\end{align*}
In particular, selecting $r := \lceil \log(2  \sqrt{k}) \rceil$ yields
\begin{align*}
    H_k - 1 &< \frac{2c_{\max}D}{c_{\min}} \lceil \log(2  \sqrt{k}) \rceil (2 \sqrt{k})^{\frac{1}{\lceil \log(2  \sqrt{k}) \rceil}} \\
    &\leq \frac{2c_{\max}D}{c_{\min}} \lceil \log(2  \sqrt{k}) \rceil \underbrace{(2 \sqrt{k})^{\frac{1}{ \log(2  \sqrt{k})}}}_{= e}. 
\end{align*}
Hence,
\begin{align*}
    \Omega_K \leq \left\lceil 6 \frac{c_{\max}}{c_{\min}} D  \log(2  \sqrt{K}) \right\rceil.
\end{align*}


\section{Proof of Lem.\,\ref{lemma_bound_F}}
\label{app:bound.failure.events}

For notational ease, in Sect.\,\ref{app:bound.failure.events} we adopt the notation $H_{k} :=H_{k,0}$, $\wt{\pi}_{k} := \wt{\pi}_{k,0}$, $\varepsilon_{k} := \varepsilon_{k,0}$ (i.e., we remove the subscript $0$).

We denote by $\tau_k$ (resp.\,$\wt{\tau}_k$) the hitting time of policy $\pi_k$ in the true model $p$ (resp.\,in the optimistic model $\wt{p}_k$). For any $h \in [H_k]$ we define
\begin{align*}
    \Gamma_{k,h}(s_{k,h}) = \mathds{1}_{\{ \tau_k(s_{k,h}) > H_k - h \}} - \mathbb{P}(\wt{\tau}_k(s_{k,h}) > H_k - h).
\end{align*}

Since $F_K = \sum_{k=1}^K \mathds{1}_{\{ \tau_k(s_{k,1}) > H_k-1 \}}$, we have
\begin{align*}
    F_K = \sum_{k=1}^K \Gamma_{k,1}(s_{k,1}) + \sum_{k=1}^K \mathbb{P}(\wt{\tau}_k(s_{0}) > H_k-1).
\end{align*}
We have for $h \in [H_k-1]$, $\mathds{1}_{\{ \tau_k(s_{k,h}) > H_k - h \}} = \mathds{1}_{\{ \tau_k(s_{k,h+1}) > H_k - h - 1\}}$ and therefore
\begin{align*}
    \Gamma_{k,h}(s_{k,h}) &= \mathds{1}_{\{ \tau_k(s_{k,h+1}) > H_k - h - 1\}} - \sum_{y \in \mathcal{S}'} \wt{p}_k(y \, \vert \,  s_{k,h}, \wt{\pi}_k(s_{k,h})) \mathbb{P}(\wt{\tau}_k(y) > H_k - h - 1) \\
    &\leq \mathds{1}_{\{ \tau_k(s_{k,h+1}) > H_k - h - 1\}} - \sum_{y \in \mathcal{S}'} p(y \, \vert \,  s_{k,h}, \wt{\pi}_k(s_{k,h})) \mathbb{P}(\wt{\tau}_k(y) > H_k - h -1) + 2 \beta_k(s_{k,h},\wt{\pi}_k(s_{k,h})) \\
    &= \Gamma_{k,h+1}(s_{k,h+1}) + \Psi_{k,h} + 2 \beta_k(s_{k,h},\wt{\pi}_k(s_{k,h})),
\end{align*}
where we define
\begin{align*}
    \Psi_{k,h} = \mathbb{P}(\wt{\tau}_k(s_{k,h+1}) > H_k - h  - 1) - \sum_{y \in \mathcal{S}'} p(y \, \vert \,  s_{k,h}, \wt{\pi}_k(s_{k,h})) \mathbb{P}(\wt{\tau}_k(y) > H_k - h - 1).
\end{align*}
Furthermore, whatever the value of $s_{k,H_k}$ we have
\begin{align*}
    \Gamma_{k,H_k}(s_{k,H_k}) = \mathds{1}_{\{ \tau_k(s_{k,H_k}) > 0\}} - \mathbb{P}(\wt{\tau}_k(s_{k,H_k}) > 0) = \mathds{1}_{\{s_{k,H_k} \neq \overline{s}\}} -  \mathds{1}_{\{s_{k,H_k} \neq \overline{s}\}} = 0.
\end{align*}
By telescopic sum we thus get
\begin{align*}
    \Gamma_{k,1}(s_{k,1}) &= \sum_{h=1}^{H_k-1} (\Gamma_{k,h}(s_{k,h}) - \Gamma_{k,h+1}(s_{k,h+1})) + \Gamma_{k,H_k}(s_{k,H_k}) \\
    &\leq \sum_{h=1}^{H_k-1} \Psi_{k,h} + 2 \sum_{h=1}^{H_k-1} \beta_k(s_{k,h},\wt{\pi}_k(s_{k,h})).
\end{align*}
Summing over the episode index $k$ yields
\begin{align*}
    F_K \leq \sum_{k=1}^K \sum_{h=1}^{H_k-1} \Psi_{k,h} + 2 \sum_{k=1}^K \sum_{h=1}^{H_k-1} \beta_k(s_{k,h},\wt{\pi}_k(s_{k,h})) + \sum_{k=1}^K \mathbb{P}(\wt{\tau}_k(s_{0}) > H_k - 1).
\end{align*}
$(\Psi_{k,h})$ is a martingale difference sequence with $\abs{\Psi_{k,h}} \leq 2$, so from Azuma-Hoeffding's inequality, in the same vein as in Eq.\,\eqref{azuma_martingale}, we have with probability at least $1-\frac{2\delta}{3}$
\begin{align*}
    \sum_{k=1}^K \sum_{h=1}^{H_k-1} \Psi_{k,h} \leq 2 \sqrt{2 \left(\sum_{k=1}^K H_k \right) \log\left(\frac{3 \left(\sum_{k=1}^K H_k \right)^2}{\delta}\right)} \leq 2 \sqrt{2 \Omega_{K} K \log\left(\frac{3 ( \Omega_{K} K)^2}{\delta}\right)}.
\end{align*}
By the pigeonhole principle (Eq.\,\ref{pigeonhole_principle}), we have
\begin{align*}
    \sum_{k=1}^K \sum_{h=1}^{H_k-1} \beta_k(s_{k,h},\wt{\pi}_k(s_{k,h})) \leq 2 S \sqrt{8 A \Omega_{K} K \log\left(\frac{2 A \Omega_{K} K}{\delta}\right)}.
\end{align*}
From Lem.\,\ref{equation_discrete_phase_type_distribution} and H{\"o}lder's inequality, we have
\begin{align*}
    \sum_{k=1}^K \mathbb{P}(\wt{\tau}_k(s_{0}) > H_k - 1) = \sum_{k=1}^K \mathds{1}_{s_0} (Q_{\wt{\pi}_k}^{\wt{p}_k})^{H_k-1} \mathds{1} \leq \sum_{k=1}^K \norm{\mathds{1}_{s_0}}_1 \norm{(Q_{\wt{\pi}_k}^{\wt{p}_k})^{H_k-1} \mathds{1}}_{\infty} \leq \sum_{k=1}^K \norm{(Q_{\wt{\pi}_k}^{\wt{p}_k})^{H_k-1}}_{\infty}.
\end{align*}
Consequently, by choice of $H_k := \min \{ n > 1 ~\vert~ \norm{(Q_{\wt{\pi}_k}^{\wt{p}_k})^{n-1}}_{\infty} \leq \frac{1}{\sqrt{k}} \}$, we get
\begin{align*}
     \sum_{k=1}^K \mathbb{P}(\wt{\tau}_k(s_{0}) > H_k - 1) \leq \sum_{k=1}^K \frac{1}{\sqrt{k}} \leq 2 \sqrt{K}.
\end{align*}


\section{Proof of Lem.\,\ref{lemma_bound_duration_second_phases}}
\label{app:proof.duration.phasetwo}

Recall that $T_{K,2}$ is the number of time steps during attempts in phase \ding{173} up to the end of environmental episode $K$. We introduce $\Omega'_{K} := \max_{k \in [K]} \max_{j \in [J_k]} H_{k,j}$ and $G_K := \sum_{k=1}^K J_k$ which is the total number of attempts in phase \ding{173} up to episode $K$. This means that $T_{K,2} \leq \Omega'_K G_K$.

First, by adapting Lem.\,\ref{lemma_bound_H_k_analysis_1} and using that in attempts in phase \ding{173} we have $c_{\max} = c_{\min} = 1$, we have under the event $\mathcal{E}$,
\begin{align}
    \Omega'_K \leq \left\lceil 6 D \log(2 \sqrt{G_K}) \right\rceil.
\label{eq_a}
\end{align}

We can decompose $G_K$ as the sum of attempts that succeed in reaching $\overline{s}$ (equal to $F_K$ which is upper bounded by Lem.\,\ref{lemma_bound_F}) and of those that fail in reaching $\overline{s}$, whose number we denote by $F_K^{\dagger}$. We then have
\begin{align}
    G_K \leq F_K + F_K^{\dagger}.
\label{eq_b}
\end{align}
By adapting Lem.\,\ref{lemma_bound_F}, we have the following high-probability bound, for any value of $G_K$,
\begin{align}
    F_K^{\dagger} 
    &= O\left( S \sqrt{A \Omega'_K G_K  \log\left(\frac{A \Omega'_{K} G_K}{\delta}\right) } \right).
\label{eq_c}
\end{align}
Plugging Eq.\,\eqref{eq_a} and \eqref{eq_b} into Eq.\,\eqref{eq_c} yields
\begin{align*}
      G_K \leq F_K + O\left( S \sqrt{A D G_K }   \log\left(\frac{A D G_K}{\delta}\right) \right).
\end{align*}
Hence we get
\begin{align*}
    G_K \leq 2 F_K \underbrace{- G_K + O\left( S \sqrt{A D G_K } \log\left(\frac{A D G_K}{\delta}\right) \right)}_{:= (y)},
\end{align*}
where $(y)$ can be bounded using the technical Lem.\,\ref{lemma_kazerouni_1} as follows
\begin{align*}
    (y) \leq O\left( S^2 A D \left[ \log\left(\frac{ S A D }{\sqrt{\delta}} \right)\right]^2 \right).
\end{align*}
Plugging in the result of Lem.\,\ref{lemma_bound_F} yields
\begin{align*}
    G_K = \wt{O}\left( S \sqrt{\frac{ c_{\max}}{c_{\min}} A D K \log\left( \frac{K}{\delta}\right)} + S^2 A D \log\left(\frac{1}{\delta}\right) \right).
\end{align*}
This bound can be translated in a bound on $T_{K,2}$ using Eq.\,\eqref{eq_a} as follows
\begin{align*}
    T_{K,2} &= O\left( D G_K \log( S \sqrt{G_K}) \right) = \wt{O}\left( D S  \sqrt{\frac{ c_{\max}}{c_{\min}} A D K \log\left( \frac{K}{\delta}\right)} \log(K) + S^2 A D^2 \log\left(\frac{1}{\delta}\right) \log(K)  \right).
\end{align*}




\section{Proof of Thm.\,\ref{theorem_regret_nonproper}}

The (possibly non-stationary) policy $\mu_k$ executed at each episode $k$ can be written as $(\wt{\pi}_{k,0}, \wt{\pi}_{k,1}, \ldots, \wt{\pi}_{k,J_k})$. As explained in Sect.\,\ref{proof_sketch}, by assigning a regret of $c_{\max}$ to each time step during attempts in phase \ding{173} (i.e., during the executions of the policies $\wt{\pi}_{k,1}, \ldots, \wt{\pi}_{k,J_k}$), we can decompose the regret as
\begin{align*}
    \Delta(\UCSSP,K) = \sum_{k=1}^K \left[ \left( \sum_{h=1}^{\tau_{k,0}} c(s_{k,h}, \mu_{k}(s_{k,h})) \right) - V^{\star}(s_0) \right] \leq \sum_{k=1}^K \left[ \left(\sum_{h=1}^{H_{k,0}} c(s_{k,h}, \wt{\pi}_{k,0}(s_{k,h})) \right)  - V^{\star}(s_0) \right] + c_{\max} T_{K,2}.
\end{align*}

Suppose from now on that the event $\mathcal{E}$ is true (this holds with probability at least $1-\frac{\delta}{3}$). Lem.\,\ref{regret_key_lemma} yields that with probability at least $1-\frac{2\delta}{3}$,
\begin{align*}
     \sum_{k=1}^K \left[ \left( \sum_{h=1}^{H_{k,0}} c(s_{k,h}, \wt{\pi}_{k,0}(s_{k,h})) \right) - V^{\star}(s_{0}) \right] &\leq  4c_{\max} D S \sqrt{8 A \Omega_K K \log\left(\frac{2 A \Omega_K K}{\delta}\right)} \\ &+ 2c_{\max}D \sqrt{2 \Omega_K K \log\left(\frac{3 (\Omega_K K)^2}{\delta}\right)} + \frac{c_{\min}}{2} \Omega_K \left(1 + \log(\Omega_K K)\right),
\end{align*}

where according to Lem.\,\ref{lemma_bound_H_k_analysis_1},
\begin{align*}
    \Omega_{K} \leq \left\lceil 6 \frac{c_{\max}}{c_{\min}} D \log(2 \sqrt{K}) \right\rceil.
\end{align*}


On the other hand, Lem.\,\ref{lemma_bound_duration_second_phases} yields 
\begin{align*}
    T_{K,2} &= \wt{O}\Bigg( D S  \sqrt{\frac{ c_{\max}}{c_{\min}} A D K \log\left( \frac{K}{\delta}\right)} \log(K) + S^2 A D^2 \log\left(\frac{1}{\delta}\right) \log(K)  \Bigg).
\end{align*}

Putting everything together finally yields that with probability at least $1-\delta$, for any $K \geq 1$,
\begin{align*}
    \Delta(\UCSSP, K) = \wt{O}\left( c_{\max} D S  \sqrt{\frac{ c_{\max}}{c_{\min}} A D K \log\left( \frac{K}{\delta}\right)} \log(K) + c_{\max} S^2 A D^2 \log\left(\frac{1}{\delta}\right) \log(K)  \right).
\end{align*}




\section{Relaxation of Assumptions}
\label{app_relax_asm}


\subsection{Straightforward extension to unknown, stochastic costs}
\label{subsection_relaxation_unknown_costs}

Although we assume (as in e.g., \citealp{azar2017minimax}) that the costs are known and deterministic for ease of exposition, we emphasize that extending the setting to unknown stochastic costs poses no major difficulty. The only requirement is that the learner needs to know in advance the range of the non-goal costs, i.e., the constants $c_{\min}$ and $c_{\max}$. In that case, at the beginning of each attempt $(k,0)$ (i.e., in phase \ding{172}), the confidence set $\mathcal{M}_{k,0}$ is not only defined with the confidence interval on the transition probabilities but also with a confidence interval on the costs. Namely, we consider
\begin{align*}
    \mathcal{M}_{k,0} := \{ \langle \mathcal{S}, \mathcal{A}, \widetilde{c}, \widetilde{p} \rangle ~\vert ~ \widetilde{p}(\cdot|s,a) \in B_{k,0}(s,a), \widetilde{c}(s,a) \in B'_{k,0}(s,a) \},
\end{align*}
where $B_{k,0}(s,a)$ is defined as in Sect.\,\ref{sect_ucssp_alg}, and where for any $a \in \mathcal{A}$, $\widetilde{c}(\overline{s},a) = 0$ while for any $s \in \mathcal{S}$,
\begin{align*}
    B'_{k,0}(s,a) := [\widehat{c}_{k,0}(s,a) - \beta'_{k,0}(s,a), \widehat{c}_{k,0}(s,a) + \beta'_{k,0}(s,a)] \cap [c_{\min}, c_{\max}],   
\end{align*}
with $\wh{c}_{k,0}(s,a)$ the empirical costs and
\begin{align*}
    \beta'_{k,0}(s,a) := 2 \sqrt{\frac{\log \left( \frac{6 S A N_{k,0}^+(s,a)}{\delta} \right)}{N_{k,0}^+(s,a)}}.
\end{align*}

The analysis on the regret bound of \UCSSP then only adds an additional error term on estimating the transition costs, which is subsumed by the other terms. Consequently, we obtain exactly the same regret bound as in Thm.\,\ref{theorem_regret_nonproper}.

\subsection{Relaxation of Asm.\,\ref{assumption_finite_SSP_diameter} (i.e., if $M$ is non-SSP-communicating, i.e., $D = + \infty$)}
\label{subsection_relaxation_asm_finite_SSP_diameter}

The requirement that the goal is reachable from any state (Asm.\,\ref{assumption_finite_SSP_diameter}) is a natural and inherent assumption of the SSP problem as introduced in \citet{bertsekas1995dynamic}. However, a reasonable extension is to allow for the existence of (potentially unknown) \textit{dead-end} states, i.e., states from which reaching the goal is impossible. In that case, $\EVI_{\SSP}$, which operates of the entire state space $\mathcal{S}$, fails to converge since the values at dead-end states are infinite. \citet{kolobov2012theory} propose to put a ``cap'' on any state's cost by optimizing the \textit{truncated value function}, or Finite-Penalty criterion,
\begin{align*}
    V_{J}^{\pi}(s) := \min \left\{ J, ~V^{\pi}(s) \right\},
\end{align*}
where $J > 0$ corresponds to a penalty incurred if a dead-end state is visited. From \citet{kolobov2012theory}, there exists an optimal policy $\pi_J^{\star}(s)$ that minimizes $V_{J}^{\pi}(s)$ and the optimal truncated value function $V_{J}^{\star}$ is a fixed point of the modified Bellman operator $\mathcal{L}_J$ defined as
\begin{align*}
    \mathcal{L}_J V(s) := \min \Big\{ J,~ \min_{a \in \mathcal{A}} \Big[ c(s, a) + \sum_{y \in \mathcal{S}} p(y| s,a) V(y) \Big] \Big\}.
\end{align*}

Denote by $\mathcal{S}^{DE} \subsetneq \mathcal{S}$ the set of dead-end states. We replace Asm.\,\ref{assumption_finite_SSP_diameter} with the following assumptions.
\begin{assumption}
    1) $s_0 \notin \mathcal{S}^{DE}$. 2) $V^{\star}(s_0) < +\infty$ and an upper bound $J$ on $V^{\star}(s_0)$ is known. 3) We augment the action space $\mathcal{A}$ with an action $\overline{a}$ that causes a transition from any state in $\mathcal{S}$ to the target state with probability $1$ and cost $J$ (i.e., we place ourselves in a resetting environment).
    \label{assumption_J}
\end{assumption}
Note that 1) and 3) of Asm.\,\ref{assumption_J} are required to make the learning problem and the definition of regret sensible (i.e., we have $V^{\star}(s_0) < +\infty$ and we have the possibility to reset whenever we are stuck in a dead-end state). Moreover, 2) guarantees that $V^{\star}(s_{0}) = V_J^{\star}(s_{0})$ and that if we run $\EVI_{\SSP}$ on $\mathcal{L}_J$ instead of $\mathcal{L}$, then $J$ is an upper bound on the optimistic value function output by $\EVI_{\SSP} $(instead of $c_{\max} D$ which is vacuous when $D = + \infty$). Note that 2) is tightly related to the requirement of \citet{fruit2018efficient} of prior knowledge on an upper bound of the span of the optimal bias function, and that 1) is similar to the assumption of a starting state belonging to the set of communicating states in \TUCRL \citep{fruit2018near}.

With those assumptions at hand, we consider the algorithm $\UCSSP\textrm{-}\mathcal{L}_J$, which differs from $\UCSSP$ in 3 ways: it iterates \EVI$_\textrm{SSP}$ on the operator $\mathcal{L}_J$, the length of the $k$-th phase \ding{172} is set to $H_k^{(J)} := 6 \frac{J}{c_{\min}} \log(2 \sqrt{k})$, and it executes action $\overline{a}$ at the end of each attempt \ding{172} (this means that there is no more phase \ding{173}, and the $k$-th attempt \ding{172} exactly corresponds to the $k$-th environmental episode). 

\begin{lemma} Under Asm.\,\ref{assumption_J} and \ref{assumption_costs}, with probability at least $1-\delta$,
\begin{align*}
\Delta(\UCSSP\textrm{-}\mathcal{L}_J, K) = O\left( J S \sqrt{ A \Omega_K^{(J)} K } \log\left(\frac{\Omega_K^{(J)} K}{\delta}\right) \right),
\end{align*}
where $\Omega_K^{(J)} := 6 \frac{J}{c_{\min}} \log(2 \sqrt{K})$.
\label{regret_J_asm_relaxed}
\end{lemma}

\begin{proof}
We have
\begin{align*}
\Delta(\UCSSP\textrm{-}\mathcal{L}_J, K) = \sum_{k = 1}^{K} \left[ \left( \sum_{h=1}^{\tau_k(s_0)} c(s_{k,h}, \wt{\pi}_k(s_{k,h})) \right) - V_J^{\star}(s_{0}) \right] \leq \sum_{k = 1}^{K} \left[ \left( \sum_{h=1}^{H_k^{(J)}} c(s_{k,h}, \wt{\pi}_k(s_{k,h})) \right) - V_J^{\star}(s_{0}) \right] + J F_K,
\end{align*} 
where the double sum can be bounded by 
\begin{align*}
     O \left( J S \sqrt{A \Omega_K K \log\left(\frac{\Omega_K K}{\delta}\right)} \right)
\end{align*}
by adapting the proof of Lem.\,\ref{regret_key_lemma}, since $\wt{\pi}_k$ is the greedy policy w.r.t.\,the optimistic value function $\widetilde{v}^{(J)}_k$ which satisfies both $\widetilde{v}^{(J)}_k(s_0) \leq V^{\star}_J(s_0)$ and $\norm{\widetilde{v}^{(J)}_k}_{\infty} \leq J$.

Note that the optimistic hitting time $\tau_{\wt{\pi}_k}^{\wt{p}_k}$ starting from any state in $\mathcal{S} \setminus \mathcal{S}^{DE}$ still follows a discrete PH distribution with $\abs{\mathcal{S}^{DE}}+1$ absorbing states (which can be reduced to a discrete PH distribution with a single absorbing state and with the same distribution of the time to absorption). Consequently, using the same reasoning as in the proof of Lem.\,\ref{lemma_bound_H_k_analysis_1}, we can prove that under the event $\mathcal{E}$,
\begin{align*}
    \mathbb{P}(\tau_{\wt{\pi}_k}^{\wt{p}_k}(s_0) \geq H_k^{(J)}) \leq \frac{1}{\sqrt{k}}.
\end{align*}
Hence we can bound $F_K$ exactly as in Lem.\,\ref{lemma_bound_F}. We obtain the desired regret bound by using that $\Omega_K^{(J)} := \max_{k \in [K]} H_K^{(J)} = 6 \frac{J}{c_{\min}} \log(2 \sqrt{K})$ by choice of $H_k^{(J)}$.
\end{proof}

Interesting future directions in the setting where $D = + \infty$ could be to attempt to remove the need for the prior knowledge $J$ (i.e., weaken Asm.\,\ref{assumption_J}), or to focus on the related problem of maximizing the probability of reaching the goal state while keeping cumulative costs low (see e.g.,~\citealp[][Sect.\,6]{kolobov2012theory}).

\subsection{Relaxation of Asm.\,\ref{assumption_costs} (i.e., if $c_{\min} = 0$)}
\label{subsection_relaxation_asm_positive_costs}

\jt{While the assumption of positive costs seems natural in numerous episodic problems and is commonly used in the SSP literature~\citep[see e.g.,][]{hansen2012suboptimality, teichteil2012stochastic}, we now consider the case where zero non-goal costs may exist, i.e., $c_{\min} = 0$. In such case, the optimal policy is not guaranteed to be proper anymore \cite{bertsekas1995dynamic}. We thus change the definition of SSP-regret and compare to the best proper policy, that is,
\begin{align}
\Delta(\mathfrak{A},K) := \sum_{k = 1}^{K} \bigg[ \Big( \sum_{h=1}^{\tau_k(s_0)} c(s_{k,h}, \mu_k(s_{k,h})) \Big) - V^\star(s_{0}) \bigg], \quad \quad \textrm{with} \quad V^{\star} := \min_{\pi \in \PSD} V^\pi, \quad \pi^{\star} \in \argmin_{\pi \in \PSD} V^\pi.
\label{eq_regret_zero_costs}
\end{align} }%
The existence of $c_{\min} > 0$ is leveraged in our analysis to bound $\Omega_K$, more specifically in Eq.\,(\ref{cost_weighted_optimism_expected_hitting_time}), which uses that the property of optimism w.r.t.\,the value functions (i.e.,~$\widetilde{v}_{k,0} \leq V^{\star}$ component-wise) yields a ``cost-weighted optimism'' w.r.t.\,the expected hitting times, i.e.,~$\mathbb{E}(\wt{\tau}_{k,0}) \leq \frac{2 c_{\max}}{c_{\min}} \mathbb{E}(\tau_{\pi^{\star}})$ component-wise. Yet if zero costs are possible (i.e.,~$c_{\min}=0$), then this implication fails to hold.

To circumvent this problem a natural idea is to introduce an additive perturbation $\eta_{k,0} > 0$ to the cost of each transition in the true SSP (note that a small offset of costs to avoid to tricky case of zero costs is also performed by \citealp{bertsekas2013stochastic}). One may hope that this would not affect the behavior of the optimal policy, yet whereas in finite- and infinite-horizon this is indeed the case (i.e.,~offsetting the costs by a positive constant does not affect the behavior of the optimal policy), Lem.\,\ref{lemma_costs_offset} shows that this property does not hold in the SSP setting.


\begin{lemma}\label{lemma_costs_offset}
For any $\eta>0$, there exists an SSP instance whose optimal policy is different from the one of an identical SSP with all of its transition costs offset by $\eta$.  
\end{lemma}

\begin{figure}[t]
    \centering
    \begin{tikzpicture}
    \node[draw, circle] (s0) at (0,0) {$s_0$};
    \node[draw, circle] (s1) at (6,0) {$\overline{s}$};
    \node[draw, circle] (s2) at (2,1) {$s_1$};
    \node[draw, circle] (s3) at (4,1) {$s_2$};
    \draw[->] (s0) -- (s1) node[midway, above] {$a_{0,0}$};
    \draw[->] (s0) -- (s2) node[midway, above] {$a_{0,1}$};
    \draw[->] (s2) -- (s3) node[midway, above] {$a_{1,0}$};
    \draw[->] (s3) -- (s1) node[midway, above] {$a_{2,0}$};
    \end{tikzpicture}
    \caption{SSP instance used in the proof of Lem.\,\ref{lemma_costs_offset}.}
    \label{fig:SSP_example_costs_offset}
\end{figure}
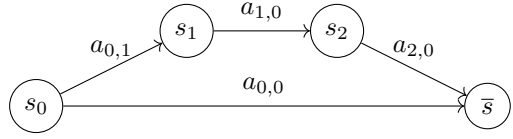

\begin{proof}
Let us consider the SSP from Fig.\,\ref{fig:SSP_example_costs_offset}, whose costs are $c(s_{0}, a_{0,0}) = 4\eta$ and $c(s_{0}, a_{0,1}) = c(s_{1}, a_{1,0}) = c(s_{2}, a_{2,0}) = \eta$. The optimal policy executes action $a_{0,0}$ in state $s_{0}$. Yet if the costs are all offset by $\eta$, the optimal policy executes action $a_{0,1}$ in state $s_{0}$.
\end{proof}

Offsetting the costs thus introduces a bias which should be adequately controlled by the choice of $\eta_{k,0}$. We consider the algorithm $\UCSSP\textrm{-}\mathcal{L}_{\eta}$, which differs from \UCSSP by introducing an additive perturbation $\eta_{k,0} > 0$ to the cost of each transition in the \textit{optimistic model} for each attempt $(k,0)$ (i.e., in phase \ding{172}), i.e., the algorithm iterates \EVI$_\textrm{SSP}$ up to an accuracy of $\varepsilon_{k,0} := \frac{c_{\max}}{t_{k,0}}$ on the operator $\mathcal{L}_{\eta}$ defined as
\begin{align*}
    \mathcal{L}_{\eta} V(s) := \min_{a \in \mathcal{A}} \Big[ c(s, a) + \eta + \sum_{y \in \mathcal{S}} p(y| s,a) V(y) \Big],
\end{align*}
where $\eta > 0$ depends on the episode $k \in [K]$. 

\begin{lemma} If $c_{\min}=0$, under Asm.\,\ref{assumption_finite_SSP_diameter} and the regret definition of Eq.\,\eqref{eq_regret_zero_costs} , by selecting $\eta_{k,0} = \frac{1}{k^{1/3}}$, we get with overwhelming probability that
\begin{align*}
\Delta(\UCSSP\textrm{-}\mathcal{L}_{\eta}, K) = ~ &\widetilde{O}\left( c_{\max} D S \sqrt{ c_{\max} D A} K^{2/3} + \Upsilon^{\star} K^{2/3} + c_{\max} D S \sqrt{\Upsilon^{\star} A K } \right. \\
& \left. + \Upsilon^{\star} S \sqrt{ c_{\max} D A} K^{1/3} + \Upsilon^{\star} S \sqrt{\Upsilon^{\star} A} K^{1/6} + S^2 A D^2 \right),
\end{align*}
where $\Upsilon^{\star} := \norm{\mathbb{E}\left[\tau_{\pi^{\star}}\right]}_{\infty}$ is the worst-case (in terms of starting state) expected hitting time of the optimal policy $\pi^{\star}$ in the original SSP (i.e., without any cost offset).
\label{regret_cmin_0_asm_relaxed}
\end{lemma}

\begin{proof}

For notational ease, throughout the proof of Lem.\,\ref{regret_cmin_0_asm_relaxed} we adopt the notation $\eta_{k} :=\eta_{k,0}$, $H_{k} :=H_{k,0}$, $\wt{\pi}_{k} := \wt{\pi}_{k,0}$, $\varepsilon_{k} := \varepsilon_{k,0}$ (i.e., we remove the subscript $0$).

$\UCSSP\textrm{-}\mathcal{L}_{\eta}$ modifies the \EVI~procedure so that it selects a pair $(\wt{\pi}_k, \wt{p}_k)$ that satisfies for any $s \in \mathcal{S}$,
\begin{align}
    (\wt{\pi}_k, \wt{p}_k) \in \argmin_{\wt{\pi}, \wt{p}} \wt{v}^{(\eta)}_{\wt{\pi},\wt{p}}(s),
\label{eq_argmin_v_tilde_eta}
\end{align}
where 
\begin{align*}
    \wt{v}^{(\eta)}_{\wt{\pi},\wt{p}}(s) := \mathbb{E}_{\widetilde{p}}\left[ \sum_{t = 1}^{\tau_{\widetilde{\pi}}(s)} c(s_{t}, \widetilde{\pi}(s_t)) + \eta_k \,\Big\vert\,s \right] = \mathbb{E}_{\widetilde{p}}\left[ \sum_{t = 1}^{\tau_{\widetilde{\pi}}(s)} c(s_{t}, \widetilde{\pi}(s_t)) \,\Big\vert \,s \right] + \eta_k \mathbb{E}_{\widetilde{p}}\left[ \tau_{\widetilde{\pi}}(s)\right],
\end{align*}
and we introduce for ease of notation $\widetilde{v}_k^{(\eta)}(s) := \wt{v}^{(\eta)}_{\wt{\pi}_k,\wt{p}_k}(s)$ and $\widetilde{v}_k(s) := \mathbb{E}_{\widetilde{p}_k}\left[ \sum_{t = 1}^{\tau_{\widetilde{\pi}_k}(s)} c(s_{t}, \widetilde{\pi}_k(s_t)) \,\Big\vert \,s \right]$. 

From Eq.\,(\ref{eq_argmin_v_tilde_eta}) we have that under the event $\mathcal{E}$, $\widetilde{v}_k^{(\eta)}(s) \leq \wt{v}^{(\eta)}_{\pi^{\star},p}(s)$, or equivalently by expanding, 
\begin{align}
    \widetilde{v}_k^{(\eta)}(s) = \widetilde{v}_k(s) + \eta_k \mathbb{E}_{\wt{p}_k}\left[ \tau_{\wt{\pi}_{k}}(s) \right] \leq \mathbb{E}_p\left[ \sum_{t = 1}^{\tau_{\pi^{\star}}} c(s_{t}, \pi^{\star}(s_t)) + \eta_k \,\Big\vert\,s \right] = V^{\star}(s) + \eta_k \mathbb{E}\left[\tau_{\pi^{\star}}(s)\right].
\label{eq_eta_opt}
\end{align}
Plugging into Eq.\,\eqref{eq_eta_opt} that $\widetilde{v}_k(s) \geq 0$ and $\norm{V^{\star}}_{\infty} \leq c_{\max} D$ from Lem.\,\ref{lemma_1} (which does not require $c_{\min} > 0$) yields 
\begin{align}
    \norm{\mathbb{E}_{\wt{p}_k}\left[ \tau_{\wt{\pi}_{k}} \right]}_{\infty} \leq \frac{c_{\max} D}{\eta_k} + \Upsilon^{\star}.
\label{ineq_last}
\end{align}

Hence the term $\frac{c_{\max}D}{c_{\min}}$ in Eq.\,\eqref{cost_weighted_optimism_expected_hitting_time} (and thus in Lem.\,\ref{lemma_bound_H_k_analysis_1}) can be replaced by the upper bound in Eq.\,\eqref{ineq_last}, which implies that under the event $\mathcal{E}$,
\begin{align*}
    \Omega_K \leq 6 \left( \frac{c_{\max} D}{\eta_K} + \Upsilon^{\star} \right) \log(S \sqrt{K}).
\end{align*}

Furthermore, using Eq.\,(\ref{eq_eta_opt}) the regret can be decomposed as
\begin{align*}
    \sum_{k = 1}^{K} \left[ \left( \sum_{h=1}^{\tau_k(s_0)} c(s_{k,h}, \wt{\pi}_k(s_{k,h})) \right) - V^{\star}(s_{0}) \right] \leq \sum_{k = 1}^{K} \left[ \left( \sum_{h=1}^{H_k} c(s_{k,h}, \wt{\pi}_k(s_{k,h})) \right) - \widetilde{v}_k^{(\eta)}(s_0) \right] + \Upsilon^{\star} \sum_{k=1}^K \eta_k + c_{\max} T_{K,2},
\end{align*}
where the double sum can be bounded by (excluding lower-order terms)
\begin{align*}
     O \left( (c_{\max} D + \eta_K \Upsilon^{\star}) S \sqrt{A \Omega_K K \log\left(\frac{\Omega_K K}{\delta}\right)} \right),
\end{align*}
by adapting the proof of Lem.\,\ref{regret_key_lemma}, since $\wt{\pi}_k$ is the greedy policy w.r.t.\,the optimistic value function $\widetilde{v}_k^{(\eta)}$ which satisfies $\norm{\widetilde{v}_k^{(\eta)}}_{\infty} \leq c_{\max} D + \eta_k \Upsilon^{\star}$ from Eq.\,(\ref{eq_eta_opt}). Moreover, we can bound $T_{K,2}$ as in Sect.\,\ref{section_general_SSP} by using Lem.\,\ref{lemma_bound_duration_second_phases}. 

Hence selecting $\eta_k = \frac{1}{k^{1/3}}$ and plugging in the bound on $\Omega_K$ yields the desired bound.
\end{proof}

An interesting future direction could be to allow for negative costs yet this extension is outside the scope of the paper.

\subsection{Summary}

We report in Table \ref{tab:regret} the regret guarantees of \UCSSP (by isolating the dependencies on $K$ and on $D$ or $J$), depending on the assumptions made (and the corresponding choices of Bellman operator for $\EVI_{\textrm{SSP}}$). We notice that if $D = + \infty$ and under Asm.\,\ref{assumption_J}, $\UCSSP\textrm{-}\mathcal{L}_J$ satisfies a regret bound where the infinite term $D$ is replaced with the known upper bound $J \geq V^{\star}(s_0)$. Moreover, $\UCSSP\textrm{-}\mathcal{L}_{\eta}$ can deal with the existence of zero costs, however the rate worsens from $\sqrt{K}$ (in Thm.\,\ref{theorem_regret_nonproper} which requires $c_{\min} > 0$) to $K^{2/3}$, due to the bias introduced by offsetting the costs in the optimistic model. Finally, it is straightforward to combine the two aforementioned variants and derive $\UCSSP\textrm{-}\mathcal{L}_{J, \eta}$ which can handle both $D=+\infty$ (under Asm.\,\ref{assumption_J}) and $c_{\min} = 0$.

\begin{table}[h!]
   \centering 
   {\renewcommand{\arraystretch}{1.6}
   \begin{tabular}{|c|c|} 
   \hline
   \textbf{Assumptions} & \textbf{Regret bound} \\
     \hline
     \textbf{$c_{\min} > 0$ (Asm.\,\ref{assumption_costs}) and $D < \infty$ (Asm.\,\ref{assumption_finite_SSP_diameter})} & $\wt{O}(D^{3/2}\sqrt{K})$ \\
   \hline
    \textbf{$c_{\min} > 0$ (Asm.\,\ref{assumption_costs}) and $V^{\star}(s_0) \leq J$ w/ \RESET (Asm.\,\ref{assumption_J})}  & $\wt{O}(J^{3/2} \sqrt{K})$ \\
    \hline
    \textbf{$c_{\min} = 0$ and $D < \infty$ (Asm.\,\ref{assumption_finite_SSP_diameter})} & $\wt{O}(D^{3/2} K^{2/3})$ \\
    \hline
    \textbf{$c_{\min} = 0$ and $V^{\star}(s_0) \leq J$ w/ \RESET (Asm.\,\ref{assumption_J})} & $\wt{O}(J^{3/2} K^{2/3})$ \\
   \hline
   \end{tabular}
   }
   \caption{Regret guarantees of \UCSSP depending on the assumptions made.}
    \label{tab:regret}
\end{table}

\section{Experiments}
\label{app:experiments}

\begin{figure}[t]
        \centering

        \includegraphics[width=.38\textwidth]{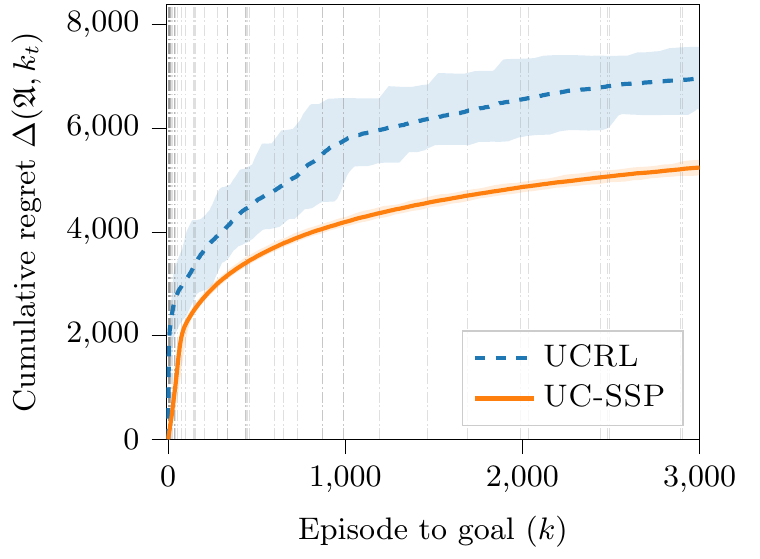} \tikz[remember picture, baseline]{\node (aa) at (0,0) {};}
        \hspace{1.4in}
        \includegraphics[width=.37\textwidth]{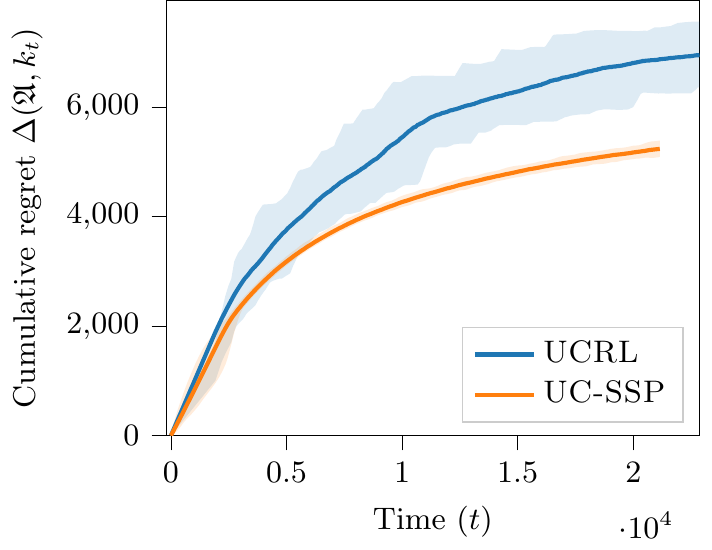}\\
        \hspace*{.1in}\includegraphics[width=.37\textwidth]{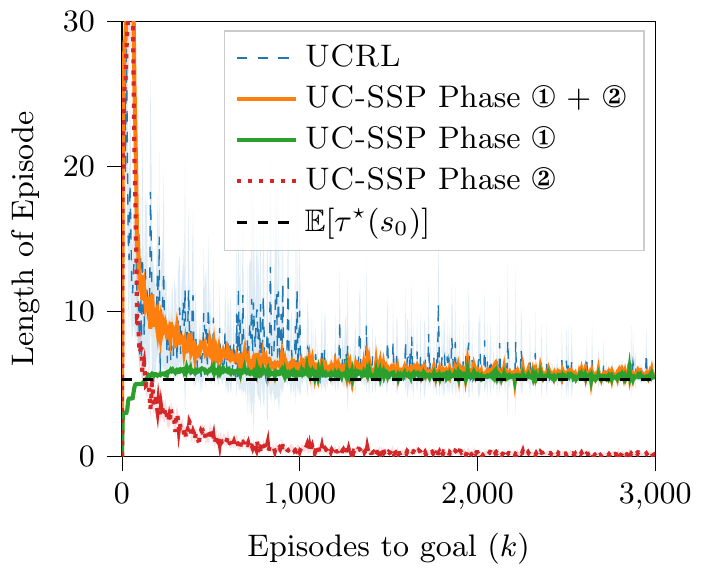}\hspace{1.4in}
        \includegraphics[width=.36\textwidth]{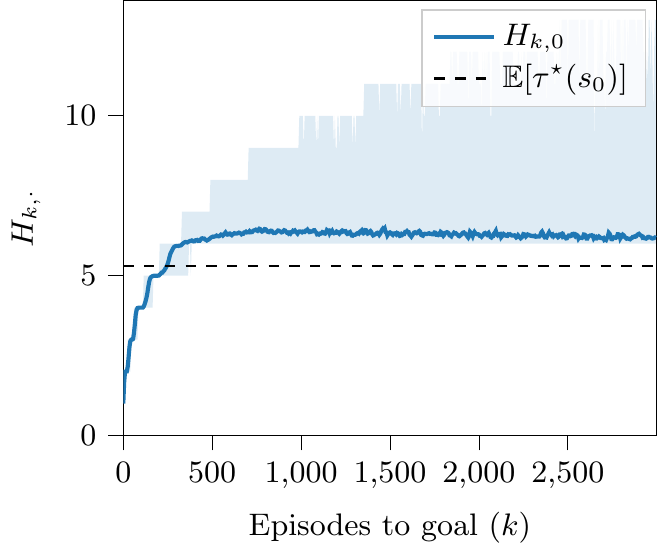}

        \caption{Comparison of \UCSSP and \UCRL in the case of uniform-cost SSP. The plots are averaged over $200$ repetitions. We report the mean and the maximum and minimum value for top line and figure bottom right. For the bottom-left figure, we report the standard deviation of the mean at 96\% to simplify the visualization.}
        \label{fig:gridworld.uniform}
\end{figure}

In this section, we empirically validate our theoretical findings and perform an ablation study of the algorithms.
We consider $3$ scenarios: 1) uniform-cost SSP; 2) SSP with $c_{\min} > 0$ and 3) SSP with $c_{\min} = 0$.
In all the experiments, we consider the same $(3 \times 4)$ gridworld but we modify the cost function.
The agent can move using the cardinal actions (Right, Down, Left, Up).
An action fails with probability $p_{f} = 0.05$. In this case (failure), the agent uniformly follows one of the other directions.
Walls are absorbing, i.e., if the action leads against the wall, the agent stays in the current position with probability 1. For example, $p((0,0) | (0,0), right) = \frac{2p_{f}}{3}$, $p((1,0) | (0,0), right) = \frac{p_{f}}{3}$ and $p((0,1) | (0,0), right) = 1 - p_{f}$.
If we consider \emph{Up}, we have $p((0,0)|(0,0), Up) = 1$.
For the experiments we used the theoretical confidence intervals without constants, i.e., $\beta_{k,j}(s,a) = \sqrt{\frac{S L}{N^+_{k,j}(s,a)}}$ with $L = \log(SAN^+_{k,j}(s,a)/0.1)$.
The remaining parameters are set as prescribed by the theory.
All the results are averaged over $200$ runs.

\textbf{1)} The first experiment aims to compare \UCRLtwo \citep{jaksch2010near} and \UCSSP in the case of uniform-cost SSP studied in Sect.\,\ref{section_uniform_costs} (see Fig.\,\ref{fig:gridworld.uniform}).
Thus we set $c(s,a) = 1$ for any $(s,a) \in \mathcal{S} \times \mathcal{A}$, and $c(\wb s,a) =0$ for all $a \in \mathcal{A}$.
We evaluate the algorithms at $K=3000$ episodes.
Fig.\,\ref{fig:gridworld.uniform}(top left) shows that the regret of both algorithms is sublinear, as stated by the theoretical analysis. Interestingly, the regret of \UCRL is higher than the one incurred by \UCSSP.
This is possibly due to algorithmic structure of \UCRL, which behaves in epochs (or algorithmic episodes) and each epoch ends when the number of visits to some state-action pair is doubled. \UCRL computes the policy only at the beginning of an epoch. As shown by the vertical lines in Fig.\,\ref{fig:gridworld.uniform}(top left), between each planning step, the agent may reach the goal multiple times.
While this can be computationally efficient, the drawback is that \UCRL may execute sub-optimal policies for long time.
On the other hand, we believe that by planning more often, \UCSSP is able to execute better policies than \UCRL. In fact, Fig.\,\ref{fig:gridworld.uniform}(bottom left) shows that the time required by \UCRL to reach the goal $\wb s$ is often higher than the one of \UCSSP.
It also shows that the length of phase \ding{173} in \UCSSP quickly goes to zero, meaning that policy executed by \UCSSP is able to quickly reach the goal.
Fig.\,\ref{fig:gridworld.uniform}(top right) shows that \UCRL requires more time (i.e., steps) than \UCSSP to successfully complete $2000$ episodes.
This test sheds light on the relationship between \UCRL and \UCSSP and shows that, despite the good regret guarantees, \UCRL may not exploit the specific structure of the SSP problem and poorly performs compared to \UCSSP.
Finally, we also plot the estimate of the hitting time computed by \UCSSP (see Fig.\,\ref{fig:gridworld.uniform}(bottom right)). As expected, it is a ``tight'' upper-bound to the expected hitting time of the optimal SSP policy ($\mathbb{E}[\tau_{\pi^\star}(s_0)] = 5.3$), except in the initial episodes where the optimistic model is far away from the true one. In the latter case, the imagined SSP problem has high probability of reaching $\wb s$ from any other state due to the high uncertainty.

\begin{tikzpicture}[remember picture,overlay]
  \node[anchor=west] at ($(aa)+(-.1in,0)$) {\includegraphics[width=4cm]{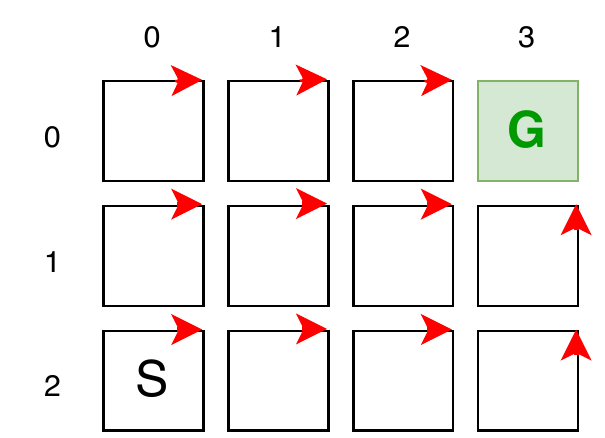}};
\end{tikzpicture}

\begin{figure}[bt]
  \centering
  \begin{minipage}{.3\textwidth}
    \centering
    \includegraphics[width=\textwidth]{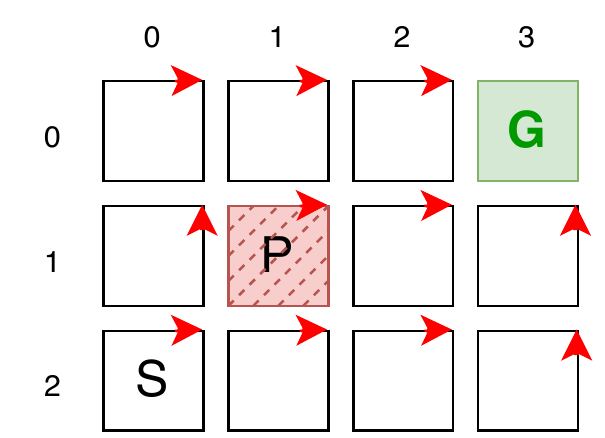}
  \end{minipage}
  \begin{minipage}{.65\textwidth}
    \centering
    \includegraphics[width=.8\textwidth]{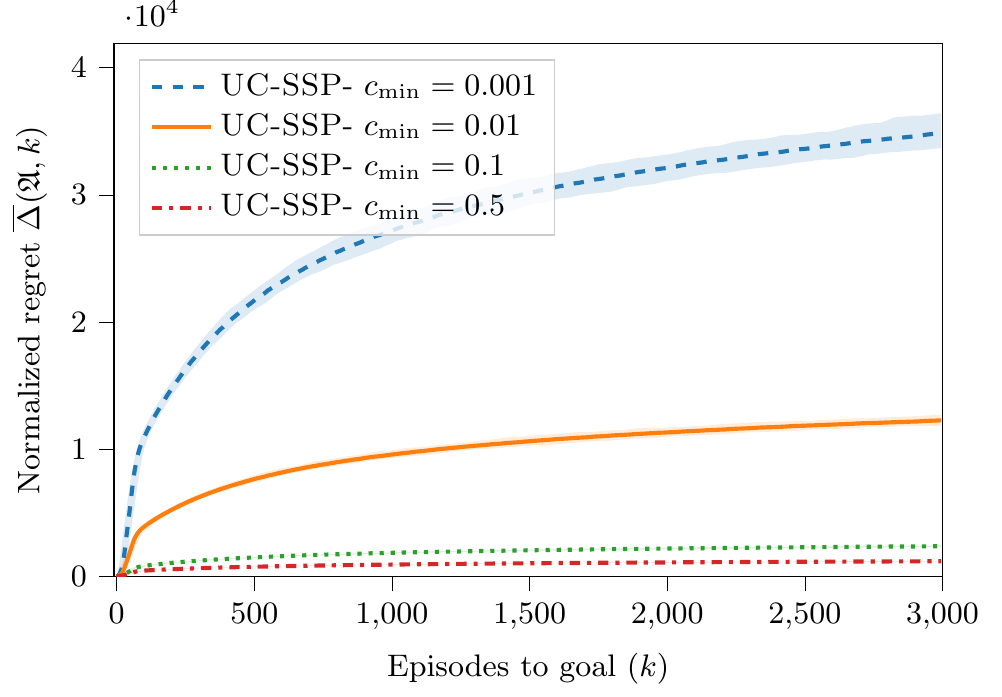}
  \end{minipage}
  \caption{Evaluation of the effect of $c_{\min} > 0$ on the regret of \UCSSP. Results are averaged over $200$ runs. We report mean value and maximum and minimum observed values.}
  \label{fig:gridworld.standard}
\end{figure}

\begin{figure}[tb]
  \centering
            \includegraphics[width=.32\textwidth]{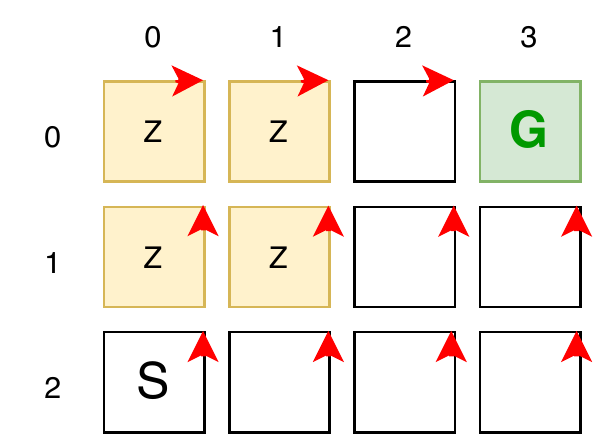}\hspace{1in}
            \includegraphics[width=.4\textwidth]{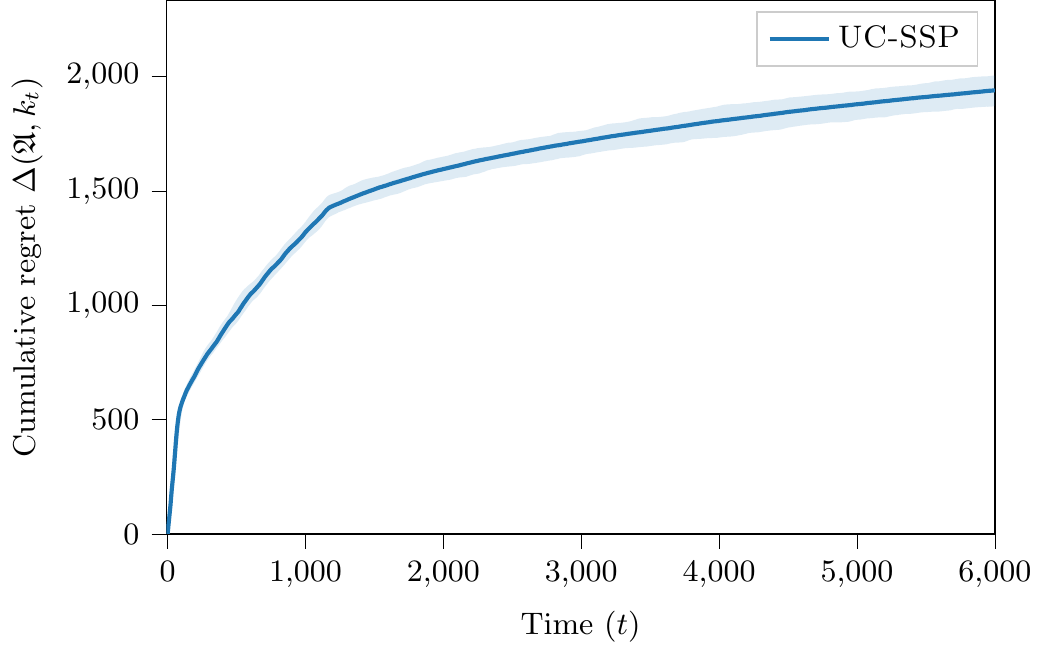}\\
            \includegraphics[width=.4\textwidth]{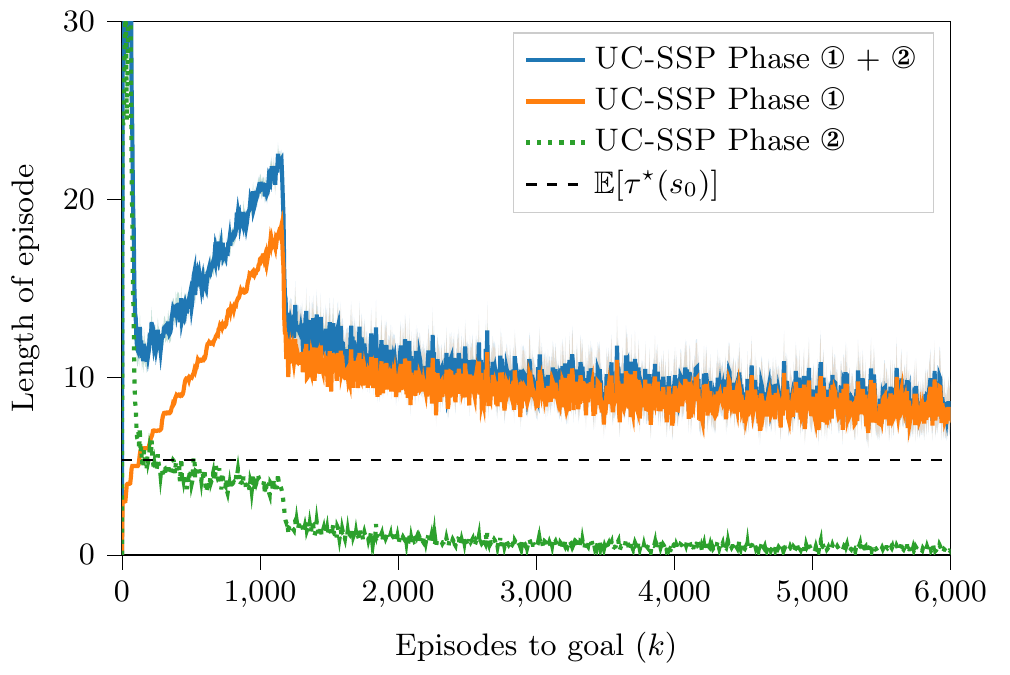}\hspace{1in}
            \includegraphics[width=.4\textwidth]{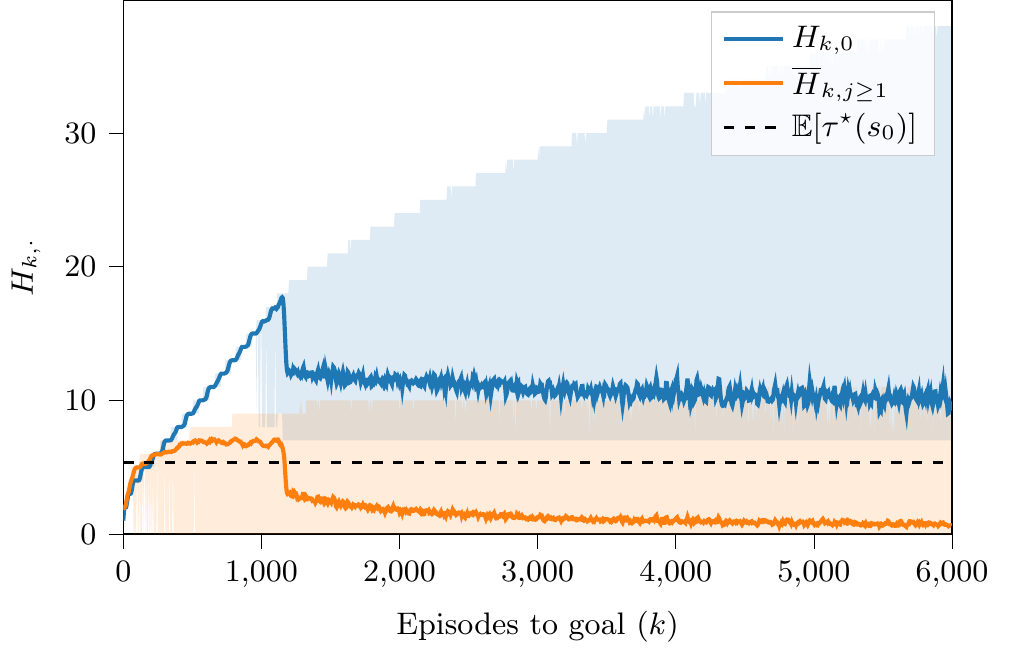}
            \caption{Evaluation of \UCSSP for $c_{\min} = 0$. See Fig.\,\ref{fig:gridworld.uniform} for details.}
  \label{fig:gridworld.zero}
\end{figure}

\vspace{-0.2in}

\textbf{2)} The second experiment focuses on non-uniform cost. At each step, the agent incurs a cost of $\beta>0$ except when in $\wt{s}=(1,1) = P$ where the cost is $1$. The state $\wt{s}$ is considered to be a sand pit and has the effect of slowing down the agent (i.e., higher cost). Formally, $c(s,a) =\beta$ for all $(s,a) \in (\mathcal{S}\setminus\{\wt s\}) \times \mathcal{A}$, $c(\wt s, a) = 1$ for all $a \in \mathcal{A}$, and $c(\wb s,a) =0$ for all $a \in \mathcal{A}$. Clearly, $c_{\min} = \beta > 0$.
Note that the optimal SSP policy is the same for all the selected values of $\beta$.
As before, we evaluate the algorithms at $K=3000$ episodes.
In Fig.\,\ref{fig:gridworld.standard}(right) we show the impact of $c_{\min}$ on the regret of \UCSSP.
First of all, we show how $c_{\min}$ affects the true solution of the SSP problem. To do so, we run VI on the true model with $\epsilon=1.e-10$ and obtain
\[
V^\star(s_0|\beta=0.5) = 2.66,~~
V^\star(s_0|\beta=0.1) = 0.55,~~
V^\star(s_0|\beta=0.01) = 0.07,~~
V^\star(s_0|\beta=0.001) = 0.02.
\]
To remove the impact of the different magnitude of the cost, we consider the normalized regret $\wb \Delta(\mathfrak{A}, K) := \frac{\Delta(\mathfrak{A}, K)}{V^\star(s_0)}$.
Fig.\,\ref{fig:gridworld.standard}(right) shows that the complexity of the learning problem scales inversely with $c_{\min}$, in the sense that the smaller $c_{\min}$ the higher the regret (i.e., the higher the learning complexity). This supports our theoretical result.

\textbf{3)} The final experiment deals with the case $c_{\min} = 0$.
We consider the states $(0,0), (0,1), (1,1)$ and $(1,0)$ to have zero cost, see Fig.\,\ref{fig:gridworld.zero}(left).
All the other states have cost defined as in experiment 2) with $\beta=0.4$. 
Note that there exists loops with zero costs, which means that there exist improper policies with finite $V$-values.
As mentioned in App.\,\ref{subsection_relaxation_asm_positive_costs}, in this case we compete against the optimal proper policy (see Fig.\,\ref{fig:gridworld.zero}(top left)).
To compute the optimal proper policy and its value $V$, we use VI with perturbation of $1e-10$~\citep{bertsekas2013stochastic}.
We evaluate the algorithms at $K=3000$ episodes.
We notice that \UCSSP has sublinear regret as expected. The perturbation of the costs has a large impact on the initial phase of \UCSSP when both uncertainty and perturbation are high. In this case, \UCSSP highly overestimates the hitting time of the optimal policy, leading to the execution of suboptimal policies for a long time (due to Phase \ding{172}). Once the perturbation and/or the uncertainty decreases, we notice that the estimated hitting time drops rapidly and approaches the true value.
It is also interesting to notice that the estimated hitting time of phase \ding{173} is never too high. This is due to the fact that phase \ding{173} aims to find the policy reaching the goal state in the smallest time.

\subsection{Bernstein Inequalities}
In this section, we provide an evaluation of the proposed algorithm with Bernstein inequalities and perform empirical comparison with later work \citep{cohen2020near}. 
Similarly to~\citep[e.g.,][]{azar2017minimax,improved_analysis_UCRL2B}, we consider the following concentration inequality of the transition probabilities: $\forall (s,a,s') \in \mathcal{S} \times \mathcal{A} \times \mathcal{S}'$,
\begin{equation}
  \left| \wt{p}(s'|s,a) - \wh p_{k,j}(s'|s,a) \right| \leq \beta_{k,j}(s,a,s') \approx \sqrt{\frac{\sigma_p^2(s,a,s') L}{N^+_{k,j}(s,a)}} + \frac{L}{N^+_{k,j}(s,a)}
\end{equation}
where $L = \log(SAN^+_{k,j}(s,a)/0.1)$ and $\sigma_p^2(s,a,s') = \wh{p}_{k,j}(s'|s,a) (1 - \wh p_{k,j}(s'|s,a))$.
Optimistic SSP planning can be performed using extended value iteration (as in Alg.\,\ref{algorithm_VI_SSP}).
We thus use the optimistic Bellman operator defined in Eq.\,\eqref{eq_bellman_operator} with $B_{k,j}(s,a) := \{ \widetilde{p} \in \mathcal{C} ~\vert ~ \widetilde{p}(\cdot \, \vert \,  \overline{s},a) = \mathds{1}_{\overline{s}}, |\widetilde{p}(s' \, \vert \,  s,a) - \widehat{p}_{k,j}(s' \, \vert \,  s,a)| \leq \beta_{k,j}(s,a,s')\}$.

\begin{figure}[bt]
  \centering
  \includegraphics[width=.55\textwidth]{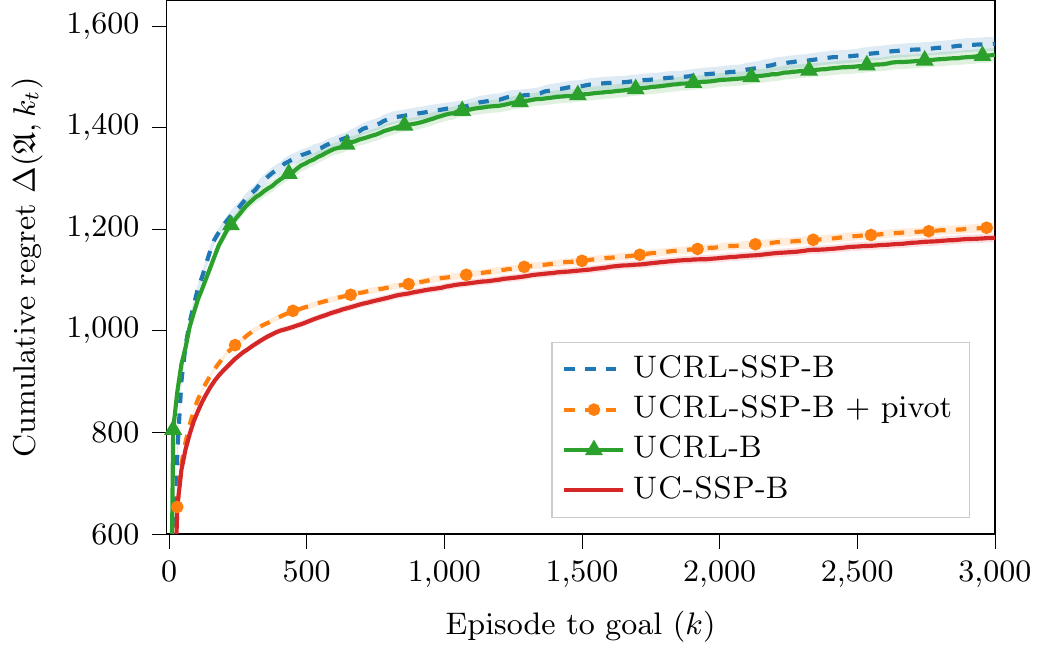}
  \caption{Evaluation of the algorithms with Bernstein inequalities and uniform cost. See Fig.\,\ref{fig:gridworld.uniform} for details. We average the results over $200$ runs and report the standard deviation of the mean at 96\%.}
\label{fig:bernstein.uniform}
\end{figure}

We compare  with \UCRLSSP~\citep{cohen2020near}. \UCRLSSP is a variant of \UCRLB~\citep{improved_analysis_UCRL2B} where the average reward planning is replaced with the SSP planning. When $c_{\min} = 0$, \UCRLSSP leverages the same perturbation idea used by \UCSSP. The cost is then defined as $c(s,a) = \max\{c(s,a), \epsilon\}$ with $\epsilon = \frac{S^2A}{K}$. 

The main goal of this section is to empirically show that, despite the  $K^{2/3}$ regret bound when $c_{\min} = 0$, \UCSSP is competitive with \UCRLSSP whose regret bound scales as $\sqrt{K}$. We also show the role of the pivot horizon used by \UCSSP. 

As done in the previous section, we start considering the uniform cost case. Fig.\,\ref{fig:bernstein.uniform} shows that \UCSSP outperforms \UCRLSSP. From Fig.\,\ref{fig:bernstein.uniform} we can see that the lower regret of \UCSSP comes from the use of the pivot horizon. Indeed, when we integrate the pivot horizon idea in \UCRLSSP\footnote{\UCRLSSP uses the same condition of \UCRLB to terminate an algorithmic episode, i.e., when the number of visits to  a state-action pair is doubled, the algorithmic episode ends. When using the pivot horizon, we simply limit the number of steps in the algorithmic episode to be at most the pivot horizon (as done for \UCSSP). We also integrated the condition of planning every time the goal state is reached but we didn't observe any significant change in this domain.} the algorithms behave similarly.
In Fig.\,\ref{fig:bernstein.uniform} we can see that \UCRLSSP behaves as \UCRLB. This is due to the fact that SSP planning is equivalent to average reward planning in this setting (i.e., uniform cost). Furthermore, it shows that, in this domain, \UCRLSSP is not able to leverage the structure of the SSP problem. In contrast, \UCSSP adapts to the SSP problem thanks to the pivot horizon.

The second experiment focuses on the case when $c_{\min} =0$.
As shown in Fig.\,\ref{fig:bernstein.zero}\emph{(left)}, \UCSSP has a low regret even in this case.
\UCRLSSP achieves the same performance of \UCSSP only when using the pivot horizon as a stopping condition of the algorithmic episode.
This shows again that the stopping condition based on pivot horizon allows the algorithms to better adapt to the the SSP structure of this problem.
Finally, Fig.\,\ref{fig:bernstein.zero}\emph{(right)} shows that phase \ding{173} happens only at the early stages of the learning process. As a consequence, \UCSSP does not suffer additional regret due to phase \ding{173} in this domain.

\begin{figure}[tb]
  \centering
  \includegraphics[width=.48\textwidth]{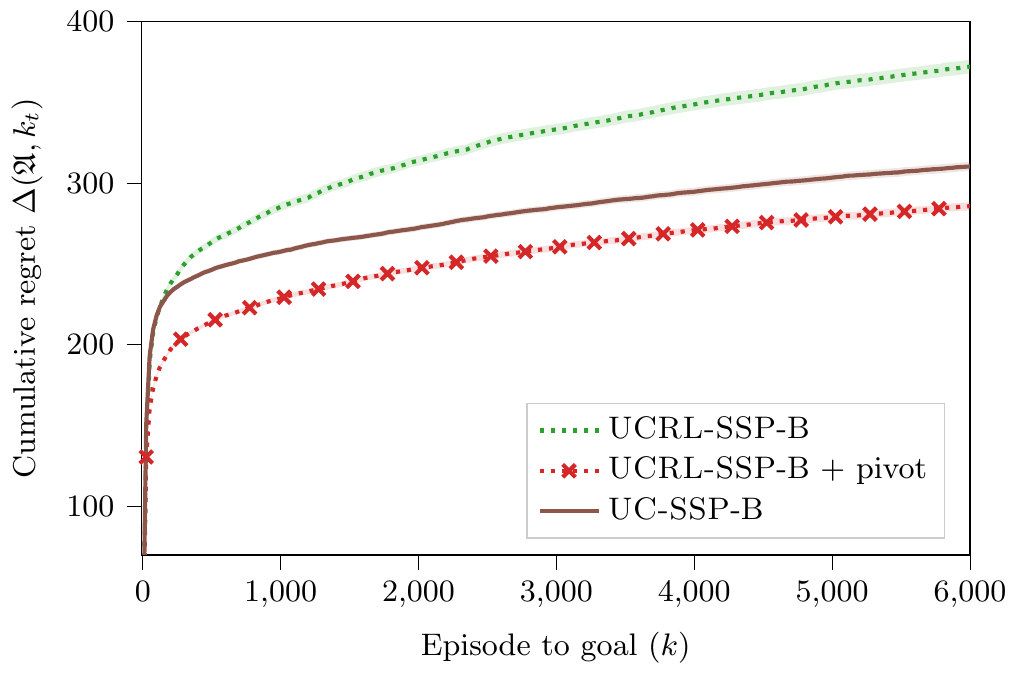}
  \includegraphics[width=.48\textwidth]{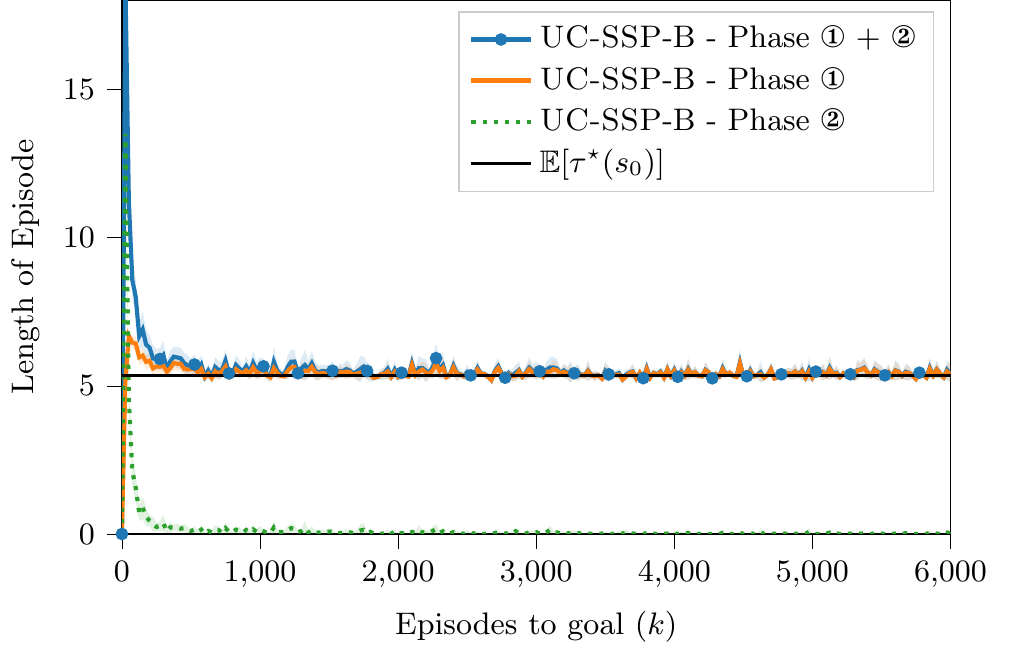}
  \caption{Evaluation of the algorithms with Bernstein inequalities and $c_{\min} = 0$. See Fig.\,\ref{fig:gridworld.zero} for details. Right figure shows the average length of Phase \ding{172} and \ding{173} for \UCSSP with Bernstein inequalities.}
\label{fig:bernstein.zero}
\end{figure}

\end{document}